\documentclass[preprint,authoryear,1p,onecolumn,10pt]{elsarticle} 

\usepackage{graphicx}
\usepackage{subcaption}
\usepackage{dblfloatfix}

\usepackage{amssymb}
\usepackage{amsmath}
\usepackage{amsthm}
\usepackage{thmtools}

\newtheorem{proposition}{Proposition}[section]

\newtheorem{definition}{Definition}[section]
\newtheorem{remark}{Remark}[section]

\journal{Journal}

\begin{document}

\begin{frontmatter}



\title{Module Extraction for Efficient Object Query \\ over Ontologies with Large ABoxes}


\author[ece]{Jia Xu\corref{cor}}
\ead{j.xu11@umiami.edu}

\author[ece]{Patrick Shironoshita}
\ead{patrick@infotechsoft.com}

\author[cs]{Ubbo Visser}
\ead{visser@cs.miami.edu}

\author[ece]{Nigel John}
\ead{nigel.john@miami.edu}

\author[ece]{Mansur Kabuka}
\ead{m.kabuka@miami.edu}

\cortext[cor]{Corresponding author.}

\address[ece]{Department of Electrical and Computer Engineering,}
\address[cs]{Department of Computer Science,\\ University of Miami, Coral Gables, FL 33146, USA}


\begin{abstract}
The extraction of logically-independent fragments out of an ontology ABox can be useful for solving the tractability 
problem of querying ontologies with large ABoxes. In this paper, we propose a formal definition of an ABox module, such that it
guarantees complete preservation of facts about a given set of individuals, and thus can be reasoned independently
w.r.t. the ontology TBox. With ABox modules of this type, isolated or distributed (parallel) ABox reasoning becomes feasible, 
and more efficient data retrieval from ontology ABoxes can be attained.
To compute such an ABox module, we present a theoretical approach and also an approximation for $\mathcal{SHIQ}$ ontologies. 
Evaluation of the module approximation on different types of ontologies shows that, on average, extracted 
ABox modules are significantly smaller than the entire ABox, and the
time for ontology reasoning based on ABox modules can be improved significantly.
\end{abstract}

\begin{keyword}


Ontology \sep Reasoning \sep ABox Module \sep $\cal SHIQ$

\end{keyword}

\end{frontmatter}


\section{Introduction}
\label{intro}

Description logics (DLs), as a decidable fragment of first-order logic, are a family of logic based 
formalisms for knowledge representation, and the mathematical underpinning for modern ontology languages such as OWL \citep{owl,Horrocks2003} and 
OWL 2 \citep{Grau2008}. A DL ontology (or knowledge base) consists of a terminological part (TBox) $\mathcal{T}$ that defines terminologies such as 
concepts and roles of a given domain, and an assertional part (ABox) $\mathcal{A}$ that describes instances of the conceptual knowledge.
Similar to a database, the ontology TBox usually represents the data schema, and the ontology ABox corresponds to the actual data set.

In recent years, DL ontologies have been increasingly applied in the development of DL-based information systems in diverse areas, 
including biomedical research \citep{Demir2010,Visser2011}, health care \citep{Bhatt2009,Iqbal2011}, decision support \citep{Haghighi2012,Lee2008}, 
and many others; see \citep{Horrocks2008} for a review of existing applications. Most of these DL applications involve 
intensive querying of the underlying knowledge that requires reasoning over ontologies. 

Standard DL reasoning services include $\it subsumption$ $\it testing$ (i.e. testing if one concept is more general than the other), and
$\it instance$ $\it checking$ (i.e. checking if an individual is one instance of a given concept). The former is considered TBox reasoning, and
the latter is considered ABox reasoning as well as the central reasoning task for information retrieval from ontology ABoxes \citep{Schaerf1994}.
These reasoning services are highly complex tasks, especially for ontologies with an expressive DL \citep{Baader2007,Tobies2001}.
For example, instance checking in a $\mathcal{SHIQ}$ ontology is in EXPTIME \citep{Tobies2001}.
Consequently, query-answering over ontologies that rely on instance checking \citep{Horrocks2000b} 
can also have high computational complexity \citep{Calvanese2007,Glimm2008,Ortiz2008}.

Existing systems that provide standard DL reasoning services include HermiT \citep{Motik2007}, Pellet \citep{Sirin2007}, 
FaCT++ \citep{Tsarkov2006}, and Racer \citep{Haarslev2001}. They are all based on the $\it (hyper)tableau$ algorithms 
that are proven sound and complete. 
Despite highly optimized implementations of reasoning algorithms in these systems, they are still confronted with the 
scalability problem in practical DL applications, where the TBox could be relatively small and manageable, the ABox 
tends to be extremely large (e.g. in semantic web setting), which can result in severe tractability problems \citep{Horrocks2004,Motik2006}.

Due to the central role that is played by ontology querying in data-intensive DL-applications, 
various approaches have been proposed to cope with this ABox reasoning problem. For example,
there are optimization strategies proposed by \cite{Haarslev2002} for instance retrieval,
and also techniques developed for ABox reduction \citep{Fokoue2006} and partition \citep{Guo2006,Wandelt2012}.
\cite{Hustadt2004} proposed to reduce DL ontology into a disjunctive datalog program, 
so that existing techniques for querying deductive databases can be reused; and 
\cite{Grosof2003}, \cite{Motik2005}, and \cite{Royer1993} have suggested the combination of ontology reasoning with inference rules.

In this paper, conceiving that a large ABox may consist of data with great diversity and isolation, we explore the modularity of an ontology 
ABox and expect ABox reasoning to be optimized by utilizing logical properties of ABox modules.
Analogous to modularity defined for the ontology TBox \citep{Grau2007a,Grau2007b}, the notion of modularity for the ABox proposed in this paper is 
also based on the semantics and implications of ontologies; and an ABox module should be a $closure$ or $encapsulation$ of logical implications 
for a given set of individuals (that is, it should capture all facts, both explicit and implicit, of the given entities), such that each module 
can be reasoned $\it independently$ w.r.t. $\mathcal{T}$.

This property of an ABox module is desired for efficient ABox reasoning under situations, either when querying information of a particular individual, 
such as instance checking (i.e. test if $\mathcal{K} \models C(a)$) or arbitrary property assertion checking (i.e. test if $\mathcal{K} \models R(a,b)$), 
or when performing instance retrieval and answering (conjunctive) queries over ontologies \citep{Horrocks2000b,Glimm2008}. For the first case, an 
independent reasoning on the ABox module for that particular
individual will be sufficient, while for instance retrieval, the property of ABox modules 
would allow the ABox reasoning to be parallelized, and thus being able to take advantage of existing parallel-processing frameworks, such as 
MapReduce \citep{Dean2008}.

For illustration, consider a real-world ontology that models and stores massive biomedical data in an ABox, and a query-answering
system that is based on this ontology. A biomedical researcher may want to obtain information of a particular gene instance, say $\tt CHRNA4$, 
to see if it is involved in any case of diseases, and in what manner. To answer such queries submitted to this ontology-based system,
a reasoning procedure is normally invoked and applied to the ontology. Note however, given the fact that in this case the 
entire ontology ABox is extremely large, complete reasoning can be prohibitively expensive \citep{Glimm2008,Ortiz2008}, and it may take 
a lengthy period to reach any conclusions. Therefore, it would be preferable to perform reasoning on a smaller subject-related (in this case, 
$\tt CHRNA4$-related) module. In particular, this $\tt CHRNA4$-related ABox module should be a closure of all facts about the individual 
$\tt CHRNA4$, so that reasoning on this module w.r.t. to $\mathcal{T}$ can achieve the same conclusions about $\tt CHRNA4$ as if the reasoning 
would be applied to the entire ontology. In addition, the reasoning time should be significantly decreased, provided the module is precise and much
smaller than the entire ABox. 
Another case of querying over this biomedical ontology could be to retrieve all genes in the ABox that belong to the same pathological class of
a certain disease. Answering such queries requires to perform instance retrieval, and a simple strategy based on the ABox module here is to 
partition the ABox based on groups of individuals and to distribute each independent reasoning task into a cluster of computers.

Our main contributions in this paper are summarized as follows:
\begin{enumerate}
\item To capture the notion of ABox modularity, we provide a set of formal definitions about the ABox module, specifying 
logical properties of an ABox module such that it guarantees preservation of sound and complete facts of a given set of individuals, and 
such that it can be independently reasoned over w.r.t. the TBox;

\item To extract such an ABox module, we develop a theoretical approach for DL $\mathcal{SHIQ}$ based on the idea of $module$-$essential$ 
assertions. This approach gives an exposition for the problem of ABox-module extraction from a theoretical perspective, allowing a better 
understanding of the problem, and providing strategies that can be built into a general framework for computation of ABox modules;

\item To cope with the complexity for checking module-essentiality using a DL-reasoner and to make the techniques applicable to practical DL 
applications, we also present a simple and tractable syntactic approximation.
\end{enumerate}

\noindent
Additionally, we present an evaluation of our approximation on different ontologies with large ABoxes, including those generated by existing 
benchmark tools and realistic ones that are used in some biomedical projects. We also present and compare several simple optimization techniques 
for our approximation. And finally, we show the efficiency of
ontology reasoning that can be achieved when using the ABox modules.

\section{Preliminaries}

\subsection{Description Logic $\mathcal{SHIQ}$}

The Description Logic $\mathcal{SHIQ}$ is extended from the well-known logic $\mathcal{ALC}$ \citep{Schmidt1991}, with added supports for role hierarchies, 
inverse roles, transitive roles, and qualified number restrictions \citep{Horrocks2000}.
A $\mathcal{SHIQ}$ ontology defines a set \textbf{R} of role names, a
set \textbf{I} of individual names, and a set \textbf{C} of concept
names. 

\begin{definition} [{\bf $\mathcal{SHIQ}$-Role}]
A $\mathcal{SHIQ}$-role can be an atomic (named) role $R \in \bf R$, or an inverse role $R^-$ with $R \in \bf R$. The complete role
set in a $\mathcal{SHIQ}$ ontology is denoted $\textbf{R}^*$=$\textbf{R} \cup \{R^- | R \in \textbf{R}\}$. 
To avoid $R^{--}$, the function Inv($\cdot$) is defined, such that Inv($R$)=$R^-$ and Inv($R^-$)=$R$.
\end{definition}

A role hierarchy $H_R$ can be defined in an ontology by a set of role inclusion axioms, each of which is expressed in the form of $R_1 \sqsubseteq R_2$, 
with $R_1, R_2 \in \bf R^*$. We call $R_1$ a $\it subrole$ of $R_2$, if $R_1 \sqsubseteq R_2 \in H_R$ or if there exist $S_1, \ldots, S_n \in \bf R^*$
with $R_1 \sqsubseteq S_1, S_1 \sqsubseteq S_2, \ldots, S_n \sqsubseteq R_2 \in H_R$. Here, $\sqsubseteq$ is reflexive and transitive. 

A role $R \in \bf R$ is $\it transitive$, denoted Trans($R$), if $R \circ R \sqsubseteq R$  or Inv($R$) $\circ$ Inv($R$) $\sqsubseteq$ Inv($R$).
Finally, a role is called $\it simple$ if it is neither transitive nor has any transitive subroles \citep{Baader2007,Horrocks2000}.

\begin{definition} [{\bf $\mathcal{SHIQ}$-Concept}]
A $\mathcal{SHIQ}$-concept is either an atomic (named) one or a complex one that can be defined using the following constructors recursively
\begin{equation*} 
\begin{split}
	C, D  ::=  & A \ |\ \neg C\ |\ C \sqcap D\ |\ C \sqcup D \\
	&|\ \forall R.C\ |\ \exists R.C\ |\ \leq nS.C\ |\ \geq nS.C 
\end{split}
\end{equation*}
where $A$ is an atomic concept, $R, S \in \textbf{R}^*$ with $S$ being simple, and $n$ is a non-negative integer. The universal concept $\top$ 
and the bottom concept $\bot$ can be viewed as $A \sqcup \neg A$ and $A \sqcap \neg A$ respectively.
\end{definition}

\begin{definition} [{\bf $\mathcal{SHIQ}$ Ontology}]
A $\mathcal{SHIQ}$ ontology is a tuple, denoted $\mathcal{K=(T,A)}$, where $\mathcal{T}$ is the terminology representing general knowledge of a 
specific domain, and $\mathcal{A}$ is the assertional knowledge representing a particular state of the terminology.

The terminology $\mathcal{T}$ of ontology $\mathcal{K}$ is the disjoint union of a finite set of role inclusion axioms (i.e. $R_1 \sqsubseteq R_2$)
and a set of concept inclusion axioms in the form of $C \equiv D$ and $C \sqsubseteq D$, where $C$ and $D$ are arbitrary $\mathcal{SHIQ}$-concepts. 
Statements $C \sqsubseteq D$ are called $\it general\ concept\ inclusion$ axioms (GCIs), and
$C \equiv D$ can be trivially converted into two GCIs as $C \sqsubseteq D$ and $D \sqsubseteq C$. 

The assertional part $\mathcal{A}$ of $\mathcal{K}$ is also known as a $\mathcal{SHIQ}$-ABox, consisting of a set of assertions (facts) about 
individuals, in the form of 

\begin{center}
\renewcommand{\arraystretch}{1.1}
\begin{tabular}{ll}
$C(a)$ & class assertion \\
$R(a,b)$ & role or property assertion \\
$a \not\approx b$ & inequality assertion
\end{tabular}
\end{center}

\noindent
where $C$ is a $\mathcal{SHIQ}$ concept, $R \in \textbf{R}^*$, and $a, b \in \textbf{I}$. 
\end{definition}

Note that explicit assertion of $a \not\approx b$ is supported in $\mathcal{SHIQ}$, while, conversely, explicit assertion of $\it equality$, 
i.e. $a \approx b$, is not supported, since its realization relies on equivalence between nominals \citep{Baader2007,FOST}, i.e. $\{a\} \equiv \{b\}$, which is illegal in $\mathcal{SHIQ}$.

\begin{definition} [{\bf $\mathcal{SHIQ}$ Semantics}]
The meaning of an ontology is given by an interpretation denoted ${\mathcal{I}}=(\Delta^{\mathcal{I}},.^{\mathcal{I}})$, 
where $\Delta^{\mathcal{I}}$ is a non-empty domain and
$.^{\mathcal{I}}$ is an $\it interpretation\ function$. 
This interpretation function $.^{\mathcal{I}}$ maps:
\begin{enumerate}
	\item every atomic concept $A \in \textbf{C}$ to a set $A^{\mathcal{I}} \subseteq \Delta^{\mathcal{I}}$, 
	\item every individual $a \in \textbf{I}$ to an element $a^{\mathcal{I}} \in \Delta^{\mathcal{I}}$ and,
	\item every role $R \in \textbf{R}$ to a binary relation on the domain, i.e. $R^{\mathcal{I}} \subseteq \Delta^{\mathcal{I}}\times\Delta^{\mathcal{I}}$.
\end{enumerate}

\noindent
Interpretation for other concepts and inverse roles are given below:
\renewcommand{\arraystretch}{1.1}
\begin{center}
\begin{tabular}{rcl}
$\top^{\mathcal{I}}$ & = & $\Delta^{\mathcal{I}}$ \\
$\bot^{\mathcal{I}}$ & = & $\emptyset$ \\
$\neg C^{\mathcal{I}}$ & = & $\Delta^{\mathcal{I}} \backslash C^{\mathcal{I}}$ \\
$(R^-)^{\mathcal{I}}$ & = & $\{(y,x)\ |\ (x,y) \in R^{\mathcal{I}}\}$ \\
$(C \sqcap D)^{\mathcal{I}}$ & = &  $C^{\mathcal{I}} \cap D^{\mathcal{I}}$ \\
$(C \sqcup D)^{\mathcal{I}}$ & = &  $C^{\mathcal{I}} \cup D^{\mathcal{I}}$ \\
$(\exists R.C)^{\mathcal{I}}$ & = & $\{x\ |\ \exists y.(x,y) \in R^{\mathcal{I}} \wedge y \in C^{\mathcal{I}} \}$ \\
$(\forall R.C)^{\mathcal{I}}$ & = & $\{x\ |\ \forall y.(x,y) \in R^{\mathcal{I}} \wedge y \in C^{\mathcal{I}} \}$ \\
$(\leq nR.C)^{\mathcal{I}}$ & = & $\{x\ |\ |\{y | (x,y) \in R^{\mathcal{I}} \wedge y \in C^{\mathcal{I}}\}| \leq n \}$ \\
$(\geq nR.C)^{\mathcal{I}}$ & = & $\{x\ |\ |\{y | (x,y) \in R^{\mathcal{I}} \wedge y \in C^{\mathcal{I}}\}| \geq n \}$ \\
\end{tabular}
\end{center}
\noindent
where $|\cdot|$ represents the cardinality of a given set.
\end{definition}

An interpretation $\mathcal{I}$ $\it satisfies$ an axiom $\alpha : C \sqsubseteq D$, 
if $C^{\mathcal{I}} \subseteq D^{\mathcal{I}}$. Such interpretation is called a $\it model$ of axiom $\alpha$. An interpretation $\mathcal{I}$ satisfies
an axiom or assertion $\alpha$:
\renewcommand{\arraystretch}{1.1}
\begin{center}
\begin{tabular}{rcl}
$R_1 \sqsubseteq R_2$ & $\it iff$ & $R_1^{\mathcal{I}} \subseteq R_2^{\mathcal{I}}$ \\
$C(a)$ & $\it iff$ & $a^{\mathcal{I}} \in C^{\mathcal{I}}$ \\
$R(a,b)$ & $\it iff$ & $(a^{\mathcal{I}}, b^{\mathcal{I}}) \in R^{\mathcal{I}}$ \\
$a \not\approx b$ & $\it iff$ & $a^{\mathcal{I}} \not\approx b^{\mathcal{I}}$ \\
\end{tabular}
\end{center}

For an ontology $\mathcal{K}$, if an interpretation $\mathcal{I}$ satisfies every axiom in $\mathcal{K}$,
$\mathcal{I}$ is a $\it model$ of $\mathcal{K}$, written $\mathcal{I} \models K$. 
In turn, ontology $\mathcal{K}$ is said $\it satisfiable$ or $\it consistent$ if it has at least one model; otherwise, it is $\it unsatisfiable$
or $\it inconsistent$, and there exists at least one contradiction in $\mathcal{K}$.

\begin{definition} [{\bf Logical Entailment}]
Given an ontology $\mathcal{K}$ and an axiom $\alpha$, $\alpha$ is called a logical entailment of $\mathcal{K}$, denoted $\mathcal{K} \models \alpha$, 
if $\alpha$ is satisfied in every model of $\mathcal{K}$.
\end{definition}

\begin{definition} [{\bf Instance checking}]
Given an ontology $\mathcal{K}$, a $\mathcal{SHIQ}$-concept $C$ and an individual $a \in \bf I$, instance checking is defined as testing whether
${\mathcal{K}} \models C(a)$ holds.
\end{definition}

Notice that instance checking is considered the central reasoning task for information retrieval from 
ontology ABoxes and a basic tool for more complex reasoning services \citep{Schaerf1994}.
Instance checking can also be viewed as ``$\it classification$'' of an individual, that is, checking
if an individual can be classified as a member of some defined DL concept. 

In most logic-based approaches, realizations of reasoning services are based on a so-called $\it refutation$-$\it style\ proof$ 
\citep{Baader2007,Horrocks2004b}, that is, to convert a inference problem to the test of ontology satisfiability. For example, to 
decide if ${\cal K} \models C(a)$, one can check if ${\cal K} \cup \{\neg C(a)\} \not\models \bot$ instead, and the answer is 
$\it yes$ iff ${\cal K} \cup \{\neg C(a)\}$ is unsatisfiable, and the answer is $\it no$ otherwise.

Ontology reasoning algorithm in current well-known systems (e.g. Pellet, HermiT, FaCT++, and Racer.) 
are based on (hyper)tableau algorithms
\citep{Haarslev2001,Motik2007,Sirin2007,Tsarkov2006} that try to
build a universal model for the given ontology based on a set of
tableau expansion rules.
For details of tableau expansion rules and description of a standard tableau algorithm for $\cal SHIQ$, we refer readers to the work in \citep{Horrocks2000}.

\subsection{Other Definitions}

We adopt notations from tableaux \citep{Horrocks2000} for referring to individuals in $R(a,b)$, such that $a$ is called an 
$R$-$predecessor$ of $b$, and $b$ is an $R$-$successor$ (or $R^-$-$predecessor$) of $a$. If $b$ is an $R$-$successor$ of $a$, 
$b$ is also an $R$-$neighbor$ of $a$.

\begin{definition}[{\bf Signature}]
Given an assertion $\gamma$ in ABox $\mathcal{A}$, the $\it signature$ of $\gamma$, denoted $Sig(\gamma)$, is defined as a set of 
named individuals occurring in $\gamma$. 
This function is trivially extensible to $Sig(\mathcal{A})$ for the set of all individuals in ABox $\mathcal{A}$.
\end{definition}

\begin{definition}[{\bf Role path}]
We say there is a $\it role\ path$ between individual $a_1$ and $a_n$, if for individuals $a_1,\ldots,a_n \in \textbf{I}$ and 
$R_1,\ldots,R_{n-1} \in \textbf{R}$, there exist either $R_i(a_i, a_{i+1})$ or $R^-_i(a_{i+1}, a_i)$ in $\mathcal{A}$ for all 
$i = 1,\ldots,n-1$. 
\end{definition}

The role path from $a_1$ to $a_n$ may involve inverse roles. For example, given $R_1(a_1, a_2)$, $R_2(a_3,a_2)$, 
and $R_3(a_3,a_4)$, the role path from $a_1$ to $a_4$ is $\{R_1, R_2^-, R_3\}$, while the opposite from $a_4$ to $a_1$ is 
$\{R_3^-, R_2, R_1^-\}$.

\begin{definition} [{\bf Simple-Form Concept}]
A concept is said to be in $\it simple\ form$, if the maximum level of nested quantifiers in the concept is less than 2. 
\end{definition}
\noindent
For example, given an atomic concept $A$, both $A$ and $\exists R.A$ are simple-form concepts, while $\exists R_1.(A \sqcap \exists R_2.A)$
is not, since its maximum level of nested quantifiers is two. Notice however, an arbitrary concept can be $\it linearly$ reduced to the simple 
form by assigning new concept names for fillers of the quantifiers, for example, $\exists R_1.\exists R_2.C$ can be converted to $\exists R_1.D$ 
by letting $D \equiv \exists R_2.C$ where $D$ is a new concept name.

\section{Definition of an ABox Module}

The notion of ABox module here is to be formalized w.r.t. ontology semantics and entailments that are only meaningful 
when the ontology is consistent. In this study, however, instead of restricting ontologies to be consistent, 
we aim to discuss the problem in a broader sense such that the theoretical conclusions hold regardless the 
consistency state of an ontology. For this purpose, we introduce the notion of $\it justifiable\ entailment$ as defined below.

\begin{definition}[{\bf Justifiable Entailment}] \label{je}
Let $\mathcal{K}$ be an ontology, $\alpha$ an axiom, and $\mathcal{K} \models \alpha$.
$\alpha$ is called a justifiable entailment of $\mathcal{K}$, iff there exists a consistent fragment $\mathcal{K}' \subseteq K$
entailing $\alpha$, i.e. $\mathcal{K}' \not\models \bot$ and $\mathcal{K}' \models \alpha$.
\end{definition}

It is not difficult to see that justifiable entailments of an ontology simply make up a subset of its logical entailments. More precisely, 
for a consistent ontology, the set of its justifiable entailments is exactly the set of its logical ones according to the definition above; 
while for an inconsistent ontology, justifiable entailments are those $\it sound$ logical ones that have consistent bases \citep{Huang2005}.
For example, given an inconsistent ontology $\mathcal{K}$=$\{C(a), \neg C(a)\}$, both $C(a)$ and $\neg C(a)$ are justifiable entailments of 
$\mathcal{K}$, but $R(a,b)$ is not.

\begin{remark}
Unless otherwise stated, we take every entailment mentioned in this paper, denoted $\mathcal{K} \models \alpha$, to mean a justifiable 
entailment. 
\end{remark}

Though complete reasoning on a large ABox may cause intractabilities, in realistic applications, a large ABox may consist of data 
with great diversity and isolation, and there are situations where a complete reasoning may not be necessary.
For example, when performing instance checking of a given individual, say instance $\tt CHRNA4$ in the biomedical ontology example
in Section \ref{intro}, the ABox contains a great portion of other unrelated matters.

An ideal solution to this ABox reasoning problem is thus to extract a subject-related module and to have the reasoning applied 
to the module instead. Particularly, to fulfill soundness and completeness, this subject-related module 
should be a closure of entailments about the given entities, which in DL, should include class and property assertions. 
This leads to our formal definition of an ABox module as follows:

\begin{definition} [{\bf ABox Module}] \label{am} 
Let $\mathcal{K=(T,A)}$ be an ontology, and a set \textbf{S} of individuals be a signature. $\mathcal{M_S}$ with ${\mathcal{M_S}} 
\subseteq {\mathcal{A}}$ is called an ABox module for signature \textbf{S}, iff for any assertion $\gamma$ (either a class 
or a property assertion) with $Sig(\gamma) \cap \textbf{S} \neq \emptyset$, $({\mathcal{T}}, {\mathcal{M_S}}) \models 
\gamma$ iff ${\mathcal{K}} \models \gamma$. 
\end{definition}

This definition provides a necessary and sufficient condition for being an ABox module. 
It guarantees that sound and complete entailments (represented by $\gamma$) of individuals in signature \textbf{S} can 
be achieved by independent reasoning on the ABox module w.r.t $\mathcal{T}$. Class assertions preserved by an ABox module 
here are limited to atomic concepts defined in $\mathcal{T}$. 

\begin{remark}
Limiting the preserved class assertions to only atomic concepts simplifies the problems to deal with in this paper.
While on the other hand, it should not be an obstacle in principle to allow this notion of ABox module to be applied for 
efficient instance retrieval or answering conjunctive queries with arbitrary concepts.
This is because we can always assign a new name $A$ for an arbitrary query (possibly complex) concept $C$ by adding axiom 
$C \equiv A$ into $\mathcal{T}$, and re-extract the ABox modules. Our
empirical evaluation shows that the overhead (time for module extraction)
could be negligible when comparing with the efficiency
gained for ABox reasoning.
\end{remark}

Definition \ref{am} does not guarantee, however, uniqueness of an ABox module for signature $\bf S$, since any super set of $\mathcal{M_S}$ 
is also a module for $\bf S$, due to the $\it monotonicity$ of DLs \citep{Baader2007}. For example, given any signature $\bf S \subseteq I$, 
the whole ABox $\mathcal{A}$ is always a module for $\bf S$. 

Note however, the objective of this paper is to extract a precise ABox module and to select only $module$-$essential$ assertions for a signature, 
so that the resulting module ensures completeness of entailments while keeping a relatively small size by excluding unrelated assertions. 
Intuitively, for assertions to be module-essential for a signature $\bf S$, they must be able to affect logical consequences of any individual 
in $\bf S$, so that by having all these assertions included, the resulting ABox module can preserve all facts of the given entities.
This criteria for being a module-essential assertion can be formalized based on the notion of $\it justification$ \citep{Kalyanpur2007} 
as given below.

\begin{definition} [{\bf Justification}] \label{j}
Let $\mathcal{K}$ be an ontology, $\alpha$ an axiom, and $\mathcal{K} \models \alpha$. We say a fragment $\mathcal{K}' \subseteq K$ is 
some justification for axiom $\alpha$, denoted $ Just(\alpha, {\mathcal{K}})$, iff ${\mathcal{K}'} \models \alpha$ 
and ${\mathcal{K}''} \nvDash \alpha$ for any $\mathcal{K}'' \subset K'$.
\end{definition} 

\begin{definition} [{\bf Module-essentiality}] \label{mea}
Let $\mathcal{K}$ be an ontology, $a$ be an individual name, and $\gamma$ an ABox assertion.
$\gamma$ is called module-essential for $\{a\}$, iff 
\begin{center} 
$\gamma \in Just(\alpha,{\mathcal{K}})\ \wedge \ {\mathcal{K}} \models \alpha$,
\end{center}
for any assertion $\alpha$ of $a$ (either $C(a)$ or $R(a,b)$, with $C \in \bf C$ and $R \in \bf R^*$ ) that can be derived from $\mathcal{K}$.
\end{definition}

A justification $Just(\alpha,{\mathcal{K}})$ for axiom $\alpha$ is in fact a $\it minimum$ fragment of the knowledge base that implies $\alpha$,
and every axiom (or assertion) in $Just(\alpha,{\mathcal{K}})$ is thus essential for this implication.
Following from this point, an assertion $\gamma$ is considered able to affect logical consequences of some individual in signature $\bf S$, 
if and only if it appears in some justification for either ($\romannumeral 1$) property assertion or ($\romannumeral 2$) class assertion of 
that individual. If so, $\gamma$ is considered module-essential for $\bf S$. 
Having all such module-essential assertions for $\bf S$ included in
$\mathcal{M_S}$, the ABox module obviously preserves all facts including both class
and property assertions of the individuals in $\bf S$. On the other
hand, with this notion of module-essentiality, the ABox module can be
kept as $\it precise$ (and as small) as possible by excluding those assertions that are
not module-essential for $\bf S$.
In this paper, we aim to compute such a precise and small ABox module 
for a given signature $\bf S$, which would ideally consist of only module-essential assertions.



In the following sections, a theoretical approach and an approximation
are presented for the computation of an ABox module. 
Without loss of generality, we assume all ontology concepts are in
simple form as defined previously, and concept terms in all class
assertions are atomic.

We will show how to compute an ABox module in two steps: we begin by showing module extraction in an $\it equality$-$\it free$ 
$\mathcal{SHIQ}$ ontology, and later we show how the basic technique
can be extended to deal with equality. Notice that, this division is only
for presentation purpose, and does not have to be considered when
using our method for module extraction in practice.

\section{ABox Modules in Equality-Free Ontologies}

An ontology is called equality-free, if 
reasoning over this ontology does not invoke any procedure for individual identification
between named individuals or named and algorithm-generated individuals.

In this section, we concentrate on a method that computes ABox modules in ontologies of this type.
To further simplify the problem, we will consider module extraction for a single individual instead of an arbitrary signature $\bf S$, 
since the $\it union$ of modules of individuals in $\bf S$ yields a module for $\bf S$, as indicated by the following proposition.

\begin{proposition} \label{p2}
Let $\bf S$ be a signature, $\mathcal{M}_{\{i\}}$ be an ABox module for each individual $i \in {\bf S}$, and
\begin{center}
${\mathcal{M_S}}=\bigcup_{i \in S} {\mathcal{M}}_{\{i\}}$.
\end{center}
$\mathcal{M_S}$ is an ABox module for $\bf S$.
\end{proposition}

\begin{proof}
Assume $\mathcal{M_S}$ is not a module for $\bf S$, then there exists an assertion $\gamma$, with $Sig(\gamma) \cap \textbf{S} \neq \emptyset$, 
and either 
$(\romannumeral 1)$ $({\mathcal{T}}, {\mathcal{M_S}}) \models \gamma$ and ${\mathcal{K}} \not\models \gamma$ or,
$(\romannumeral 2)$ ${\mathcal{K}} \models \gamma$ and $({\mathcal{T}}, {\mathcal{M_S}}) \not\models \gamma$. 

$(\romannumeral 1)$ clearly contradicts DL monotonicity. For $(\romannumeral 2)$, let individual $a \in Sig(\gamma) \cap \textbf{S}$,  
${\mathcal{M}}_{\{a\}}$ with ${\mathcal{M}}_{\{a\}} \subseteq {\mathcal{M_S}}$ be the module for $a$. Then, by the definition of a module, we have 
$({\mathcal{T}}, {\mathcal{M}}_{\{a\}}) \models \gamma$, which again conflicts with the monotonicity of DLs, since ${\mathcal{M}}_{\{a\}}$
is subsumed by ${\mathcal{M_S}}$. Hence, the proposition holds.
\end{proof}

\subsection{Strategy}

Basically, to compute an ABox module as precise as possible for a given individual $a$, we have to test every assertion in $\mathcal{A}$ to see if it is module-essential 
for $a$. That is, for every assertion we have to test whether it contributes in deducing any property assertion or class assertion of individual $a$. 

As a fact of $\mathcal{SHIQ}$, deducing class assertions (classification) of an individual usually depends on both of its class and property 
assertions. While conversely, the tableau-expansion procedure for $\mathcal{SHIQ}$ indicates that deduction of a property assertion between 
different named individuals in $\mathcal{A}$ should not be affected by their class assertions, except via $\it individual\ equality$ \citep{Horrocks2000} 
as discussed in Section \ref{modequ}. This is consistent with the tree-model property of the description logic that, if nominals are not involved in the 
knowledge base, no tableau rules can derive connection (i.e. property assertion) from individual $a$ to an arbitrary named individual, except to itself in the presence of (local) reflexivity \citep{Motik2009}. 


Therefore, given an equality-free ontology, the above observation allows us to deduce property assertions from the ABox w.r.t. role 
hierarchy and transitive roles defined in $\mathcal{T}$ \citep{Horrocks2000}, 
and a strategy for extracting an ABox module can then be devised.

\begin{proposition} \label{p3}
Given an equality-free $\mathcal{SHIQ}$ ontology, computation of an ABox module for individual $a$ can be divided into two steps:
\begin{enumerate}
\item compute a set of assertions (denoted ${\mathcal{M}}_{\{a\}}^{\mathcal{P}}$) that preserves any property assertion $R(a,b)$ of $a$,
	with $a \not\approx b$ and $R \in \bf R^*$,
\item compute a set of assertions (denoted ${\mathcal{M}}_{\{a\}}^{\mathcal{C}}$) that preserves any class assertion $C(a)$ of $a$, with $C \in \bf C$.
\end{enumerate}
\end{proposition}

For simplicity, we call $\mathcal{M}_{\{a\}}^{\mathcal{P}}$ a $\it property$-$\it preserved$ $\it module$ and ${\mathcal{M}}_{\{a\}}^{\mathcal{C}}$ a 
$\it classification$-$\it preserved$ $\it module$ for $\{a\}$.

\subsection{Property-Preserved Module}
\label{appm}
A property-preserved module of individual $a$ is essentially a set of assertions in ontology $\mathcal{K}$, which affect the deduction of $a$'s 
property assertions. This set is denoted ${\mathcal{M}}_{\{a\}}^{\mathcal{P}}$ such that
\begin{center}
${\mathcal{M}}_{\{a\}}^{\mathcal{P}} = \{\gamma\ |\ \gamma \in Just(R(a,b),{\mathcal{K}})\ \wedge \ {\mathcal{K}} \models R(a,b)\} $, 
\end{center}
where $\gamma$ is an ABox assertion, and $Just(R(a,b),{\mathcal{K}})$
is any justification for $R(a,b)$, with $a \not\approx b$.

In an equality-free $\mathcal{SHIQ}$ ontology, since property assertions between different individuals can be deduced from the ABox $\mathcal{A}$ 
w.r.t. role hierarchy and transitive roles, the computation for ${\mathcal{M}}_{\{a\}}^{\mathcal{P}}$ is then straightforward based on the following fact:
for any $R \in \textbf{R}^*$, if ${\mathcal{K}} \models R(a,b)$ with $a\not\approx b$, there are two possibilities on an equality-free 
$\mathcal{SHIQ}$ ontology, according to \citep{Horrocks2000}:
\begin{enumerate}
	\item assertion $R_0(a,b)$ or $Inv(R_0)(b,a) \in {\mathcal{A}}$ with $R_0 \sqsubseteq R$,
	\item assertions involved in a role path from $a$ to
          $b$, with all roles having a common transitive parent $R_0$
          and $R_0 \sqsubseteq R$.
          e.g. $R_1(a, a_1)$, $R_2(a_2, a_1)$, $R_3(a_2, b) \in {\mathcal{A}}$, with $R_1, R_2^-, R_3 \sqsubseteq R_0$ and $R_0$ is
          transitive.
\end{enumerate}

\noindent
Abstracting from a particular $R_0$, these two possibilities can be generalized into a formal criteria to select assertions to include in 
${\mathcal{M}}_{\{a\}}^{\mathcal{P}}$ for individual $a$:
\begin{description}
	\item[$C1.$] property assertions in $\mathcal{A}$ that have $a$ as either subject or object, or
	\item[$C2.$] property assertions in $\mathcal{A}$ that are involved in a role path from $a$ to
          some $b$, with all roles in the path having a common transitive parent.
\end{description}

\begin{proposition} 
\label{p4}
On an equality-free $\mathcal{SHIQ}$ ontology, the set of property assertions satisfying criteria $C1$ or $C2$ forms a property-preserved module 
${\mathcal{M}}_{\{a\}}^{\mathcal{P}}$ for individual $a$.
\end{proposition}

\begin{proof}
Correctness of this proposition can be verified by observation of the
tableau-constructing procedure for $\mathcal{SHIQ}$ presented in \citep{Horrocks2000}. Let a tableau $T=({\Delta}, {\mathcal{L,E}}, {.^{\mathcal{I}}})$ 
be an interpretation for $\mathcal{K}$ as defined in \citep{Horrocks2000},
where $\Delta$ is a non-empty set, $\mathcal{L}$ maps each element in $\Delta$ to a set of concepts, $\mathcal{E}$ maps each role to a set of pairs of 
elements in $\Delta$, and $.^{\mathcal{I}}$ maps individuals in $\mathcal{A}$ to elements in $\Delta$.

They have proven 
that, for tableau $T$ to be a model for $\mathcal{K}$, if ${\mathcal{K}} \models R(a,b)$,
there must be either $(a^{\mathcal{I}}, b^{\mathcal{I}}) \in {\mathcal{E}}(R)$ or a
path $(a^{\mathcal{I}}, s_1), (s_1, s_2),\ldots, (s_n, b^{\mathcal{I}})$ $\in {\mathcal{E}}(R_0)$ with $R_0 \sqsubseteq R$ and $R_0$ being transitive. 
The second scenario is consistent with criteria $C2$; while for the first one, i.e. 
$(a^{\mathcal{I}}, b^{\mathcal{I}}) \in {\mathcal{E}}(R)$, 
there are only two possibilities according to the tableau-constructing procedure: ($\romannumeral 1$) 
$R_0(a,b)$ or $R_0^-(b,a) \in {\mathcal{A}}$ with $R_0 \sqsubseteq R$ that triggers initialization of ${\mathcal{E}}(R_0)$; ($\romannumeral 2$) $R_0(a,b)$ or $R_0^-(b,a)$ 
is obtained through the ${\leq}_r$-rule for identical named individuals \citep{Horrocks2000} with $R_0 \sqsubseteq R$. ($\romannumeral 1$) 
reflects exactly the criteria $C1$, while ($\romannumeral 2$) does not apply here for equality-free ontologies. 
Therefore, the proposition holds.
\end{proof}

\subsection{Classification-Preserved Module}

To compute a precise ABox module ${\mathcal{M}}_{\{a\}}$, we need to further decide a set of 
assertions that affect classifications of the individual, and this set is denoted ${\mathcal{M}}_{\{a\}}^{\mathcal{C}}$ such that
\begin{center}
${\mathcal{M}}_{\{a\}}^{\mathcal{C}} = \{\gamma\ |\ \gamma \in Just(A(a),{\mathcal{K}})\ \wedge \ {\mathcal{K}} \models A(a)\}$, 
\end{center}
where $\gamma$ is an ABox assertion, $A$ is an atomic concept, and $Just(A(a),{\mathcal{K}})$ is any justification for $A(a)$. 

As previously stated, in $\mathcal{SHIQ}$ an individual is usually classified based on both its class and property assertions in $\mathcal{A}$.
It is obvious that explicit class assertions of $a$ form an indispensable part of ${\mathcal{M}}_{\{a\}}^{\mathcal{C}}$. 
Then, to decide any property assertion of $a$ that affects its classification, we examine each assertion captured in ${\mathcal{M}}_{\{a\}}^{\mathcal{P}}$. 

The decision procedure here is based on the idea that instance checking is reducible to concept subsumption \citep{Donini1992,Donini1994,Nebel1990}, i.e.
\begin{center}
\begin{tabular}{p{0.7\textwidth}}
{\it Given an ontology $\mathcal{K}=(T,A)$, an individual $a$ and a $\mathcal{SHIQ}$-concept $C$,
$a$ can be classified into $C$, if the concept behind $a$'s assertions in the ABox is subsumed by $C$ w.r.t. $\mathcal{T}$. }
\end{tabular}
\end{center}

\noindent
This idea automatically lends itself as a methodology, such that to determine any assertion of an individual that contributes to 
its classification, we have to decide if the concept behind this assertion is subsumed by some concept w.r.t. $\mathcal{T}$.

Consider the following example: Let an ontology $\mathcal{K}=(T,A)$ be
\begin{center}
$(\{\exists R_0.B \sqsubseteq A\},\ \{R_0(a,b), B(b)\})$,
\end{center} 
and let us ask whether $R_0(a,b)$ is essential for individual $a$'s classification or not. To answer that, we need to decide 
if the concept behind this property assertion is subsumed by some named
concept w.r.t. $\mathcal{T}$, i.e. to test if  
\begin{equation} \label{e3} 
\mathcal{K} \models \exists R.C_1 \sqsubseteq C_2
\end{equation}
for some named concept $C_2$, with $R_0 \sqsubseteq R$ and $C_1(b)$
entailed by the ontology. 
It is easy to see (\ref{e3}) is satisfied in this example by substituting $B$ for $C_1$ and $A$ for $C_2$.
We can thus determine $R_0(a,b)$ is one of the causes for the entailment $C_2(a)$ (i.e. it is in some justification 
for $C_2(a)$), and should be an element of ${\mathcal{M}}_{\{a\}}^{\mathcal{C}}$. 
Moreover, assertions in $Just(C_1(b),{\mathcal{K}})$ should also be elements of ${\mathcal{M}}_{\{a\}}^{\mathcal{C}}$, since $C_1(b)$ 
is another important factor to the classification $C_2(a)$.

The above example illustrates a simple case of a single property assertion affecting classification of an individual. 
Classification of an individual can also be 
caused by multiple assertions. For example, let ${\mathcal{K}}$ be 
\begin{center}
$(\{\exists R_0.B \sqcap \exists R_1.C \sqsubseteq A\}$, 

$\{R_0(a,b), R_1(a,c), B(b), C(c))\})$.
\end{center}
Here, $R_0(a,b)$ is still essential for the deduction $A(a)$. But when testing the subsumption in (\ref{e3}) for $R_0(a,b)$, it
will be found unsatisfied.

Thus, in order to comprehensively and completely include other information about the individual, condition (\ref{e3}) should be 
generalized into: 
\begin{equation} \label{et}
 \mathcal{K} \models \exists R.C_1 \sqcap C_3 \sqsubseteq C_2 
\end{equation}
where all other information of individual $a$ is summarized and incorporated into a concept $C_3$ with $C_3 \not\sqsubseteq C_2$. 

Moreover, taking the number restrictions in $\mathcal{SHIQ}$ into consideration, condition (\ref{et}) can be further generalized as:
\begin{equation} \label{e4}
 \mathcal{K} \models \ \geq nR.C_1 \sqcap C_3 \sqsubseteq C_2 \ \wedge \ |R^{\mathcal{K}}(a, C_1)| \geq n,
\end{equation}
where $C_2$ is a named concept, $C_1(b)$ and $C_3(a)$ are entailed by $\mathcal{K}$,
$\exists R.C_1$ is only a special case of $\geq nR.C_1$, and $R^{\mathcal{K}}(a, C_1) = \{b_i \in {\bf I} \ |\ {\mathcal{K}} 
\models R(a,b_i) \wedge C_1(b_i)\}$ 
denotes the set of distinct $R$-neighbors of individual $a$ in $C_1$.

In general, condition (\ref{e4}) indicates that in an equality-free $\mathcal{SHIQ}$-ontology, for any property assertion $R(a,b)$ to affect classification of individual $a$, the
corresponding concept must be subsumed by some named one, and 
for any (qualified) number restrictions, the number of distinct $R$-neighbors of $a$ should be no less than the cardinality required. 
With this condition derived, we are now in a position to present a procedure for computation of ${\mathcal{M}}_{\{a\}}^{\mathcal{C}}$ 
and also an ABox module for individual $a$, which is summarized in Figure \ref{f1}.

\begin{figure}[t]
\centering
\resizebox{0.9\textwidth}{!}{
\begin{tabular}{p{0.9\textwidth}}
\hline 
\begin{enumerate}
\item Compute a property-preserved module ${\mathcal{M}}_{\{a\}}^{\mathcal{P}}$ for the given individual $a$, by following the criteria $C1$ and $C2$ 
given in Section \ref{appm}. 

\item Add all explicit class assertions of $a$ into ${\mathcal{M}}_{\{a\}}^{\mathcal{C}}$.

\item For every $R_0(a,b)$ ($R_0 \in \bf R^*$) captured in ${\mathcal{M}}_{\{a\}}^{\mathcal{P}}$ and any $R$ with $R_0 \sqsubseteq R$, test if the 
corresponding condition (\ref{e4}) is satisfied. 
If it is yes, add $R_0(a,b)$, assertions in some $Just(C_1(b),{\mathcal{K}})$, and any inequality assertions between individuals in $R^{\mathcal{K}}(a, C_1)$ 
into ${\mathcal{M}}_{\{a\}}^{\mathcal{C}}$. 
		
\item Unite the sets, ${\mathcal{M}}_{\{a\}}^{\mathcal{P}}$ and ${\mathcal{M}}_{\{a\}}^{\mathcal{C}}$, to form an ABox module for $a$. 
\end{enumerate} \\ 
\hline
\end{tabular}}
\caption{Steps for computation of an ABox module for individual $a$.}
\label{f1}
\end{figure}

\subsection{An Approximation for Module Extraction}

Computation of ${\mathcal{M}}_{\{a\}}^{\mathcal{P}}$ depends on the complete role hierarchy, which should be computable using a DL-reasoner, 
since roles in $\mathcal{SHIQ}$ are atomic and, most importantly, the size of $\mathcal{T}$ should be much smaller than 
$\mathcal{A}$ in realistic applications \citep{Motik2006}. On the other hand, it is difficult to compute ${\mathcal{M}}_{\{a\}}^{\mathcal{C}}$, 
since it demands computation of both concept subsumption (i.e. condition (\ref{e4})) and justifications for class assertions 
(i.e. $Just(C_1(b),{\mathcal{K}})$). Simple approximations for both are given in this section as follows.

\begin{definition} [{\bf Approximation of (\ref{e4})}] \label{p8}
A syntactic approximation for condition (\ref{e4}) for $R_0(a,b)$ is that: to test if $\mathcal{K}$ contains any formula in the form as listed below:
\begin{equation} \label{e5}
\begin{split}
	& \exists R.C_1 \bowtie C_3 \sqsubseteq C_2 \\ 
	& |\ C_1 \sqsubseteq \forall R^-.C_2 \bowtie C_3 \\
	& |\ \geq nR.C_1 \bowtie C_3 \sqsubseteq C_2 \ \wedge  \ |\{b_i\ |\ R(a,b_i)\in {\mathcal{A}}\}| \geq n \\
\end{split}
\end{equation}
where $R_0 \sqsubseteq R$, $C_i \in \bf C$, $\bowtie$ is a place holder for $\sqcup$ and $\sqcap$.
Also note the following equivalences:
\begin{center}
\begin{tabular}{rcl}
$\exists R.C \sqsubseteq D$ & $\equiv$ & $\neg D \sqsubseteq \forall R.\neg C$   \\
$\geq nR.C$ $\sqsubseteq D$  & $\equiv$ & $\neg D \sqsubseteq\ \leq (n-1)R.C$.
\end{tabular}
\end{center}
\end{definition}

For assertion $R_0(a,b)$, the approximation for condition (\ref{e4}) is to check if any formula in $\mathcal{K}$ is
in the form of any listed axioms in (\ref{e5}). If it is yes, $R_0(a,b)$ may $\it potentially$ affect some logical entailment 
of individual $a$, and related assertions will be added into $a$'s ABox module to ensure preservation of this potential entailment.
Validity of this approximation is shown by the following proposition.

\begin{proposition} \label{p9}
Let $\mathcal{K=(T,A)}$ be a $\mathcal{SHIQ}$ ontology with simple-form
concepts only, $\geq n R_0.B$ and $C$ be $\mathcal{SHIQ}$ concepts, and $D$ a 
named concept. If 
\begin{equation} \label{sem-cond}
\mathcal{K} \models\  \geq n R_0.B \sqcap C \sqsubseteq D
\end{equation}
with $C \not\sqsubseteq D$, there must exist some formula in $\mathcal{T}$ in the form as listed in
(\ref{e5}) for some $R$ with $R_0 \sqsubseteq R$.
\end{proposition}

\begin{proof}
Since $\exists R.C_1$ is a special case of $\geq nR.C_1$, 
$\exists R.C_1 \sqsubseteq C_2$ is equivalent with $C_1 \sqsubseteq \forall R^-.C_2$, 
and together with those equivalences mentioned above, every role restriction in $\mathcal{T}$ can be converted to the form of $\geq nR.C_1$ 
by axiom manipulation. Then, the task here is reducible to proving that if (\ref{sem-cond}) is satisfied, there must be some formula in $\mathcal{K}$ 
in the form of $\geq nR.C_1\bowtie C_2 \sqsubseteq C_3$ for some $R$ with $R_0 \sqsubseteq R$.

It is straightforward that, if such no $R$ with $R_0 \sqsubseteq R$ is used in concept definition, $\geq nR.B$ is not comparable (w.r.t. subsumption) 
with any atomic concept (except $\top$ and its equivalents).
On the other hand, if $\geq nR.C_1$ is used in concept definition but occurs only in the right-hand side (r.h.s.) of GCIs,  it is unable 
to indicate any atomic concept as its subsumer, which can be confirmed by observation of a tableau-constructing procedure.

Let $P,Q$ be two atomic concepts, $Q \neq \top$, $P$ and $\neg P$ not fillers in any restrictions, and all concepts of $\mathcal{T}$ in NNF.
Assume ($\ast$) $P$ occurs only in r.h.s. of GCIs (or $\neg P$ in l.h.s.), and there is a consistent fragment $\mathcal{T}' \subseteq T$ that 
${\mathcal{T}'} \models P \sqsubseteq Q$. It follows that: 
\begin{description}
\item [$\rm (E1)$] ${\mathcal{T}'} \cup \{\top(a)\} \models \neg P \sqcup Q (a)$ for any individual $a$, since $P \sqsubseteq Q$ implies $\top \sqsubseteq \neg P \sqcup Q$.
\item [$\rm (E2)$] ${\mathcal{T}'} \cup \{\neg Q(a)\} \models \neg P(a)$, because of (E1).
\item [$\rm (E3)$] ${\mathcal{T}'} \cup \{P \sqcap \neg Q(a)\} \models \bot$, the so-called refutation-style proof for $P \sqsubseteq Q$.
\end{description}

\noindent
(E1) implies that, in any tableau that is a model of ${\mathcal{T}'} \cup \{a\}$, there must be either $\neg P$ or $Q$ in ${\mathcal{L}}(a^{\mathcal{I}})$
(i.e. the class set of $a^{\mathcal{I}}$ in the tableau), 
which can be shown by contradiction: suppose ${\mathcal{I}}_1$ is a model for ${\mathcal{T}'} \cup \{a\}$, where neither 
$\neg P$ nor $Q$ is in ${\mathcal{L}}(a^{{\mathcal{I}}_1})$. Let ${\mathcal{I}}_2$ be another tableau such that ${\mathcal{I}}_2$ coincides with ${\mathcal{I}}_1$ 
except ${\mathcal{L}}(a^{{\mathcal{I}}_2})$ is extended with $\{P,\neg Q\}$, and ${\mathcal{I}}_2$ should be clash-free
since both $P$ and $Q$ are atomic and no tableau rules can be applied. Thus, ${\mathcal{I}}_2$ turns out to be a model for ${\mathcal{T}'} \cup \{P \sqcap \neg Q(a)\}$ 
that violates (E3).

Analogously for (E2), there must be $\neg P$ in ${\mathcal{L}}(a^{\mathcal{I}})$ for any model of ${\mathcal{T}'} \cup \{\neg Q(a)\}$. Nevertheless, if $P$ occurs only in r.h.s. 
of GCIs in $\mathcal{T}$, $\neg P$ can never exist after the NNF transformation of axioms in $\mathcal{T}$, and since $P$ and $\neg P$ are not fillers in any restrictions,
${\mathcal{L}}(a^{\mathcal{I}})$ can never comprise $\neg P$ according to the tableau rules \citep{Horrocks2000}. 
Hence, the original assumption ($\ast$) does not hold.

The above case essentially indicates that, if an atomic concept occurs only in the r.h.s. of GCIs in $\mathcal{T}$, its subsumer is undecidable.
And the same general principle applies, if we consider all $\geq nR.C_1$ for any $R$ with $R_0 \sqsubseteq R$ as a single unit.
Thus, there must be some $\geq nR.C_1$ with $R_0 \sqsubseteq R$ occurring in l.h.s. of GCIs in $\mathcal{T}$, if (\ref{sem-cond}) is true.
\end{proof}

Proposition \ref{p9} shows the completeness of the approximation (\ref{e5}).
We can thus conclude that an ABox module resulting from this approximation is still able to capture complete classifications (w.r.t. $\mathcal{T}$) 
of the given individual, which are derivable from its property assertions. The following statement is an immediate consequence of this conclusion:
if $C_1 \neq \top$, an approximation for the set of assertions in some
$Just(C_1(b),{\mathcal{K}})$ is ${\mathcal{M}}_{\{b\}}$, an ABox module 
for $b$, which is computed using the same strategies described here.

Procedure for this approximation is then summarized in Figure \ref{fig5}.

\begin{figure}[t]
\centering
\resizebox{0.9\textwidth}{!}{
\begin{tabular}{p{0.9\textwidth}}
\hline 
\begin{enumerate}
\item Compute a property-preserved module ${\mathcal{M}}_{\{a\}}^{\mathcal{P}}$ for the given individual $a$, by following the criteria $C1$ and $C2$ 
given in Section \ref{appm}. 

\item Add all explicit class assertions of $a$ into ${\mathcal{M}}_{\{a\}}^{\mathcal{C}}$.

\item For every $R_0(a,b)$ captured in ${\mathcal{M}}_{\{a\}}^{\mathcal{P}}$ and any $R$ with $R_0 \sqsubseteq R$, test if $\mathcal{K}$ contains 
any formula in the form as listed in (\ref{e5}). If it is yes, add $R_0(a,b)$, all assertions in ${\mathcal{M}}_{\{b\}}$, and any inequality assertions 
between $a$'s $R$-neighbors into ${\mathcal{M}}_{\{a\}}^{\mathcal{C}}$.

\item Unite the sets, ${\mathcal{M}}_{\{a\}}^{\mathcal{P}}$ and ${\mathcal{M}}_{\{a\}}^{\mathcal{C}}$, to form an ABox 
module for $a$. 
\end{enumerate} \\ 
\hline
\end{tabular}}
\caption{Steps of the approximation for module extraction for individual $a$.}
\label{fig5}
\end{figure}

\section{Module Extraction with Equality}
\label{modequ}

In this section, we show how the outcome from the previous section can be utilized to tackle module extraction 
with individual equality.


In $\mathcal{SHIQ}$, individual equality stems from the at-most number
restriction (i.e. $\leq n.R.C$) \citep{Horrocks2000}, such that, if
individual $x$ belongs to a concept $\leq n.R.C$ while it has more than
$n$ entailed $R$-neighbors in $C$, then at least two of these $R$-neighbors
could be identical. 

This means that, the determination of individual equality requires computation of both property and class assertions of 
related individuals. 
Besides, the strategy proposed in Proposition \ref{p3} becomes infeasible, since property assertions can 
be derived from equalities (e.g. given $y \approx z$, $R(y,w)$ simply implies $R(z,w)$). 
In other words, with equality, the computation of ${\mathcal{M}}_{\{a\}}^{\mathcal{P}}$ may be dependent on that of ${\mathcal{M}}_{\{a\}}^{\mathcal{C}}$.

To address this, we present a procedure for extraction of ABox modules, which first modularizes the ABox as if it were equality-free, 
and then resolves equality through post-processing.

\begin{proposition} \label{p10}
Let individual $x \in \ \leq nR.C$ have more than $n$ entailed $R$-neighbors in
$C$, of which two, $y$ and $z$, can be determined to be equal (i.e. $y \approx z$).
Let signature $\bf S$ consist of $x$ and all its $R$-neighbors in $C$, and
\begin{center}
${\mathcal{M_S}} = \bigcup_{i \in S} {\mathcal{M}}_{\{i\}}$
\end{center}
where ${\mathcal{M}}_{\{i\}}$ is an ABox module for each individual $i \in \bf S$ in the ``equality-free'' ABox.
${\mathcal{M_S}}$ preserves the equality $y \approx z$.
\end{proposition}

Proposition \ref{p10} suggests a strategy to retain equality between $y$ and $z$, by combining ``modules'' of related individuals.
With $y \approx z$ preserved, $\mathcal{M_S}$ automatically preserves all facts of $y$ and $z$ that are derived from the equality.
Subsequently, for neighbors of $y$ and $z$ (i.e. individuals in $Sig({\mathcal{M}}_{\{y\}}^{\mathcal{P}})$ and $Sig({\mathcal{M}}_{\{z\}}^{\mathcal{P}})$), 
modules of these entities should be combined with the $\mathcal{M_S}$ obtained above, so that their facts derivable from $y \approx z$ 
can also be captured. This strategy to retain equality in ABox modules should be applied $\it recursively$ for all identities.

Notice, however, that the strategy in Proposition \ref{p10} is based on the condition that individuals $y$ and $z$ be known to be equal in the first place, 
which cannot be assured without ABox reasoning.
Nevertheless, conceiving that equality in $\mathcal{SHIQ}$ stems from number restrictions \citep{Horrocks2000},
a simple approximation for it is given in the definition below.

\begin{definition} \label{p11}
Let $x$,$y$ and $z$ be named individuals, $y,z$ be $R$-neighbors of $x$, and $y \not\approx z$ not hold 
explicitly\footnote{Either $y \not\approx z$ is not explicitly given, or assertions $C(y)$ and $\neg C(z)$ do not occur simultaneously in ${\mathcal{A}}$.},
$y$ and $z$ are considered \textbf{potential equivalents}, if their $R$-predecessor $x$:
\begin{enumerate}
\item has no potential equivalents, and has $m$ $R$-neighbors with $m \gneq n$, or

\item has a set of potential equivalents, denoted $X$ (which includes $x$), and there exists a set $S \in {X \choose {m'-n'+1}}$, such that
\begin{equation} \label{e6}
\underset{S}{max} \ |\{y_i\ |\ R(x_i,y_i) \in {\mathcal{A}} \ \wedge \ x_i\in S\}|=m \gneq n.
\end{equation} 
\end{enumerate}
where $R$ is used in number restrictions as in the axiom listed in (\ref{e5}), and $n$ is the minimum of the set $\{k \ | \geq (k+1)R.C\ in\ l.h.s.\ of\ GCIs\}$.
${X \choose {m'-n'+1}}$\footnote{If $(m'-n'+1) \geq |X|$, it denotes $\{X\}.$} denotes the set of all $(m'-n'+1)$-combinations 
of set $X$, and variables $m'$ and $n'$ are for identification of $x$ that correspond to $m$ and $n$ above, respectively.
\end{definition}

\begin{figure}[t]
\centering
\includegraphics[width=0.8\textwidth]{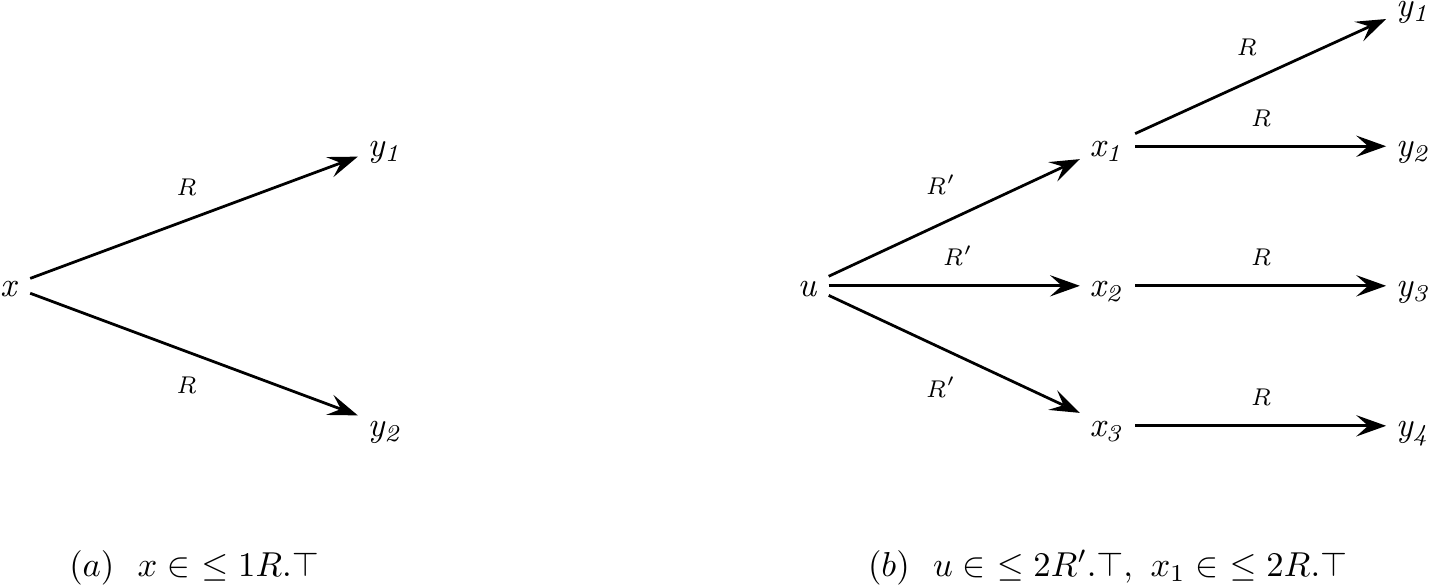}
\caption{Examples for two cases for individual $x$ in Definition \ref{p11}. In
figure (a), $y_1$ and $y_2$ are inferred to be equal, if $y_1 \not\approx y_2$ doesn't hold. 
In (b), all $x_i$(s) are considered potential equivalents if none
of them are explicitly asserted different from each other; on the other hand,
only two of them can be inferred to be equal, when sufficient
information is provided. If $x_1$ is equated with any one of the
others, its $R$-neighbors (i.e. $y_i$(s)) may be further equated.
}
\label{fig3}
\end{figure}

In this definition, for ${\mathcal{K}} \models\ \leq nR.C(x)$ to be possible it is 
required that $\leq nR.C$ or its isoform occur in the r.h.s. of GCIs in $\mathcal{T}$ 
(or $\geq (n+1)R.C$ in l.h.s. as in (\ref{e5})),
proof for which is similar to the one given for Proposition \ref{p9}.

Moreover, if $x$ itself has potential equivalents, individuals $y$ and $z$ should be elements of $\{y_i\ |\ R(x_i,y_i) \in {\mathcal{A}}\ \wedge \ x_i\in X\}$.
For the counting of potential $R$-neighbors of $x$, instead of taking the entire set $X$, only the 
$(m'-n'+1)$-combination of $X$ that maximizes the counting is considered (see (\ref{e6})), since given $m'$ $R'$-neighbors, 
the maximum possible number of decidable identical entities is $(m'-n'+1)$ out of $m'$ according to \citep{Horrocks2000}.
Examples for two cases of individual $x$ in Definition \ref{p11} are illustrated in Figure \ref{fig3}.

\begin{proposition} \label{p12}
Applying the strategy in Proposition \ref{p10} to potential equivalents generates modules that preserve individual
equality. 
\end{proposition}

\begin{proof}
(Sketch) We prove by induction.
\begin{description}
\item[Basis:] Assume initially, $y \approx z$ is the only equality in ABox that is entailed from the ${\leq}_r$-rule, then,
there must be $R_1(x,y)$ and $R_2(x,z)$ in ${\mathcal{M}}_{\{x\}}^P$, for some $x \in$ $\leq nR.C$, $x$ having more than $n$ $R$-neighbors, 
and $R_1, R_2 \sqsubseteq R$. According to the strategy in Proposition \ref{p10}, modules in ``equality-free'' ABox of $x$ and its $R$-neighbors are
merged, and resulted $\mathcal{M_S}$ thereby preserves all facts (except those derived from $y \approx z$) of these individuals, 
including assertion $R(x,i)$ and classification for each $x$'s $R$-neighbor $i$, which are sufficient to entail $y \approx z$, according to the
${\leq}_r$-rule \citep{Horrocks2000}.

\item[Inductive step:] Assume ABox $\mathcal{A}$ with arbitrary individual equality is modularized, and all extracted modules are in compliance with 
Definition \ref{am}. $\mathcal{A}$ is then further extended by adding a new assertion $R(x', y')$ that causes fresh equality between individual $y'$ and $z'$. 
If $x'$ has no equivalent, this simply follows what we have discussed above. 
Otherwise, we take into account all $x'$'s potential equivalents in set $X'$ for its $R$-neighbors that comprise both $y'$ and $z'$.
There exists a set $S \subseteq X'$ of true positives such that $|\{y_i'|R(x_i',y_i'), x_i'\in S\}|$ is greater than the cardinality restriction 
as required for individual identity, and $S$ is at most a $(m'-n'+1)$-combination of $X'$. Modules of $x'$ and all its $R$-neighbors are then
merged, and as discussed above equality between $y'$ and $z'$ is preserved. 
\end{description}
\end{proof}

\section{Related Work}
\label{relatedwork}

Modularization of the terminological part of an ontology has already been addressed by \citep{Grau2007a,Grau2007b}, where the author 
provided a well-defined idea of ontology modularity and developed a logic framework for extracting modules for terminological concepts.
On the other hand, for ABox reduction or partition, the problem has been addressed mainly by \citep{Fokoue2006}, \citep{Guo2006}, 
\citep{Du2007} and \citep{Wandelt2012}.

The idea of \citep{Fokoue2006} is to reduce large ABoxes by combining similar individuals (w.r.t. individual description) to develop 
a summary ABox as the proxy, which is small enough for scalable reasoning. This ABox summary can be useful in some scenarios that involve
general testing of consistency of an ontology, but it has a limited capability to support ontology queries such as instance retrieval.

In \citep{Guo2006}, the authors proposed a method for computing independent ABox partitions for $\mathcal{SHIF}$ ontologies, such that 
complete reasoning can be achieved if combining results of independent reasoning on each partition. This method for ABox 
partitioning is based on a set of inference rules in \citep{Royer1993}, and it encapsulates
assertions that are antecedents of a inference rule as an ABox partition. This technique is also used by \citep{Williams2010} 
to handle reasoning on large data sets.

ABox partition is also used in \citep{Du2007}, where the authors proposed an algorithm for $\mathcal{SHIQ}$ ontologies.
Based on the technique in \citep{Hustadt2004} that converts DL ontologies to disjunctive datalog programs, their algorithm further 
converts the disjunctive datalog program into a plain datalog program, to generate rules that can be employed as guidelines for ABox partition.

Both approaches above for ABox partition, however, failed to impose any logical restrictions on a single partition. An immediate consequence 
is that every single assertion can turn out to be a partition if the ontology is too simple, i.e. has no transitive roles nor concepts 
defined upon restrictions. Besides, to get complete entailments of an individual, one still has to reason over all partitions of the 
ontology.

The work most related to our proposed techniques is presented in \citep{Wandelt2012}, which focuses on $\mathcal{SHI}$ ontologies.
In their paper, instead of imposing any specification on a single ABox module, the authors provided a formal definition directly for ABox 
modularization that extends the notion of ABox partition from \citep{Guo2006}. A briefly summarized version of their definition is given below.

\begin{definition} [{\bf ABox Modularization} \citep{Wandelt2012}] 
Given ontology $\mathcal{K}=(T,A)$, an ABox modularization $M$ is defined as a set of ABoxes $\{ {\mathcal{A}}_1, \ldots, {\mathcal{A}}_n \}$, where each 
${\mathcal{A}}_i \subseteq {\mathcal{A}}$ is an ABox module. $M$ is said sound and complete for instance retrieval if, given any class assertion $C(a)$ 
($C$ is atomic), $\exists {\mathcal{A}}_i \in M$ that ${\mathcal{T}} \cup {\mathcal{A}}_i \models C(a)$ iff ${\mathcal{K}} \models C(a)$.
\end{definition}

In \citep{Wandelt2012}, the general idea for ABox modularization is based on connected components in an ABox graph. 
The authors presented an $\it ABox$-$\it Split$ based approximation, such that after applying ABox-split on each $R(a,b) \in {\mathcal{A}}$, every 
connected component in the resulted ABox graph forms an ABox module, and the set of all connected components forms the modularization 
defined above.

More precisely, this ABox-split technique checks every property assertion $R(a,b)$ in an ABox and will replace it with two generated ones, 
$R(a',b)$ and $R(a,b')$, if all the conditions below are satisfied for $R(a,b)$ \citep{Wandelt2012}:
\begin{enumerate}

\item $R$ is neither transitive nor has any transitive parent.

\item For every $\it C \in extinfo_{\mathcal{T}}^{\forall}(R)$, it satisfies that: ($\romannumeral 1$) $C \equiv \bot$, or ($\romannumeral 2$) 
there exists $D(b) \in {\mathcal{A}}$ such that ${\mathcal{K}} \models D \sqsubseteq C$, or ($\romannumeral 3$) there exists $D(b) \in {\mathcal{A}}$ such 
that ${\mathcal{K}} \models D \sqcap C \sqsubseteq \bot$.

\item For every $\it C \in extinfo_{\mathcal{T}}^{\forall}(R^-)$, it satisfies that: ($\romannumeral 1$) $C \equiv \bot$, or ($\romannumeral 2$) 
there exists $D(a) \in {\mathcal{A}}$ such that ${\mathcal{K}} \models D \sqsubseteq C$, or ($\romannumeral 3$) there exists $D(a) \in {\mathcal{A}}$ such 
that ${\mathcal{K}} \models D \sqcap C \sqsubseteq \bot$.
\end{enumerate}
\noindent
where $\it extinfo_{\mathcal{T}}^{\forall}(R)$ is the set of classes in $\mathcal{T}$ that are used as fillers of $R$ and are able to propagate through 
the $\forall$-rule \citep{Horrocks2000} in the tableau algorithm.

In general, this ABox-split based approximation will replace an original property assertion $R(a,b)$ with two generated ones, if it can be determined
that, either ($\romannumeral 1$) $R(a,b)$ has no influence on hidden implications of $a$ nor $b$, or ($\romannumeral 2$) its influence relies on 
individual's classification that has already been explicitly given. This approximation hence ``splits'' assertion $R(a,b)$ and separates $a$ and 
$b$ into different connected components (ABox modules) in ABox graph. 
For example, given ontology ${\mathcal{K}}=\{R(a,b), \forall R.B(a), B(b)\}$, $B(b)$ is entailed by $\mathcal{K}$ by either the explicitly given assertion 
or $\{R(a,b), \forall R.B(a)\}$. Nevertheless, since $B(b)$ is explicitly given, this approximation can separate individual $a$ and $b$ into two 
ABox modules.

To sum up, the method by \citep{Wandelt2012} aims at computing ABox modules as small as possible (different from the defined Exact ABox Modules in this paper) 
for $\mathcal{SHI}$-ontologies; and it makes use of information from the class hierarchy (including concept subsumption and concept disjointness, 
where complex classes may be involved) to rule out assertions for duplicate or impossible implications, which thus requires invocation of a DL reasoner.
Moreover, this approach requires a consistent ontology as the prerequisite.

\section{Empirical Evaluation}

\begin{table}[t]
\caption{ABox modules in different ontologies}
\label{results}
\setlength{\tabcolsep}{.9mm}
\begin{center}
\resizebox{\textwidth}{!}{
\begin{tabular}{c|c|c|c|c|c|c}
\hline
Ontology & Exp. & \#Ast. & \#Ind. & \shortstack{Max. \\ \#Ast./\#Ind.} & \shortstack{Avg. \\ \#Ast./\#Ind.} & \shortstack{Avg. \\ Extraction Time}  
\\ 

\hline
LUBM-1 & $\mathcal{SHI}$ & 67,465 & 17,175 & 2,921/593 & 13.1/2.4 & 1.12 ms 
\\

\hline
LUBM-2 & $\mathcal{SHI}$ & 319,714 & 78,579 & 2,921/593 & 14.4/2.5 & 1.22 ms 
\\

\hline
VICODI & $\mathcal{A}LH$ & 53,653 & 16,942 & 8,591/1 & 5.3/1 & 0.28 ms  
\\

\hline
AT & $\mathcal{SHIN}$ & 117,405 & 42,695 & 54,561/10,870 & 6.9/1.7 & 1.06 ms 
\\

\hline
CE & $\mathcal{SHIN}$ & 105,238 & 37,162 & 49,914/9,315 & 7.1/1.7 & 0.60 ms 
\\

\hline
DBpedia$^\star\_1$ & $\mathcal{SHIQ}$ & 402,062 & 273,663 & 94,862/18,671 & 2.9/1.1 & 0.62 ms
\\

\hline
DBpedia$^\star\_2$ & $\mathcal{SHIQ}$ & 419,505 & 298,103 & 160,436/17,949 & 2.8/1.1 & 0.59 ms
\\

\hline
DBpedia$^\star\_3$ & $\mathcal{SHIQ}$ & 388,324 & 255,909 & 140,050/35,720 & 3.1/1.2 & 1.80 ms
\\

\hline
DBpedia$^\star\_4$ & $\mathcal{SHIQ}$ & 398,579 & 273,917 & 139,422/18,208 & 2.9/1.1 & 0.51 ms
\\

\hline
\end{tabular}}
\end{center}
\end{table}

We implemented our approximation using Manchester's OWL API\footnote{http://sourceforge.net/projects/owlapi},
and evaluate it on a lab PC with Intel(R) Xeon(R) 3.07 GHz CPU, Windows 7, and 1.5 GB Java heap.
For test data, we collected a set of ontologies with large ABoxes: 
\begin{enumerate}
\item VICODI\footnote{http://www.vicodi.org} is an ontology that
  models European history with a simple TBox;
\item LUBM(s) are well-known benchmark ontologies generated using tools provided by \citep{Guo2005}; 
\item Arabidopsis thaliana (AT) and Caenorhabditis elegans (CE) are
  two biomedical 
ontologies\footnote{http://www.reactome.org/download} based on a
complex TBox\footnote{http://www.biopax.org} that models biological 
pathways; and
\item DBpedia$^\star$ ontologies are acquired from the original DBpedia ontology \footnote{http://wiki.dbpedia.org/Ontology?v=194q} \citep{Auer2007}.
They have a common terminological part $\mathcal{T}$, DL expressivity
of which is extended from $\mathcal{ALF}$ to $\mathcal{SHIQ}$ by
adding transitive 
roles, role hierarchy, and concepts defined on role restrictions, and their ABoxes are generated by randomly sampling the original ABox.
\end{enumerate}

\noindent
Details of these ontologies are summarized in Table \ref{results}, in terms of expressiveness (Exp.), number of assertions (\#Ast.), 
and number of individuals (\#Ind.).

\subsection{Evaluation of Extracted Modules}

\begin{table}[t]
\caption{Small and simple ABox modules in ontologies}
\label{smallabox}
\begin{center}
\resizebox{0.8\textwidth}{!}{
\begin{tabular}{c|c|c|c}
\hline
Ontology & Total \#Modules & \shortstack{\#Module with \\ \#Ast $\leq$ 10 (\%)} & \shortstack{\#Module with \\ Signature Size = 1 (\%)} \\

\hline
LUBM-1 & 7,264 & 7,210 (99.3) & 7,209 (99.2) \\

\hline
LUBM-2 & 31,361 & 30,148 (96.1) & 31,103 (99.2) \\

\hline
VICODI & 16,942 & 16,801 (99.2) & 16,942 (100) \\

\hline 
AT & 24,606 & 23,464 (95.4) & 23,114 (93.9) \\

\hline
CE & 21,305 & 20,010 (93.9) & 19,455 (91.3) \\

\hline
DBpedia$^\star\_1$ & 239,758  & 233,231 (97.3) & 237,676 (99.1) \\

\hline
DBpedia$^\star\_2$ & 264,079 & 257,555 (97.5) & 261,973 (99.2) \\

\hline
DBpedia$^\star\_3$ & 208,401 & 201,905 (96.9) & 206,377 (99.0) \\

\hline
DBpedia$^\star\_4$ & 241,451 & 234,903 (97.3) & 239,346 (99.1) \\

\hline
\end{tabular}}
\end{center}
\end{table}

\begin{figure}[t]
\centering
\includegraphics[width=0.7\textwidth]{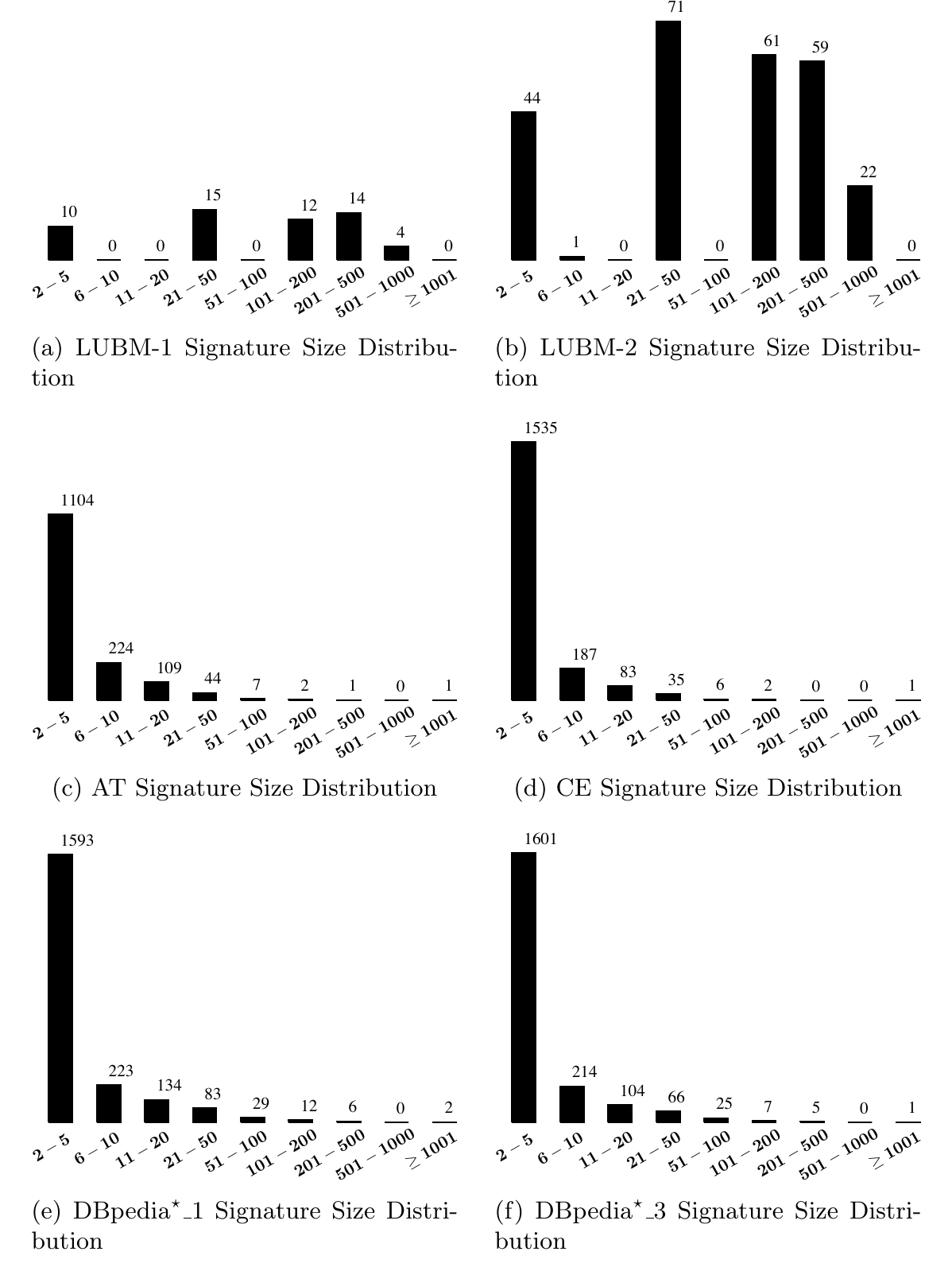}
\caption{Distributions of signature sizes of ABox modules. 
X-axis is the size range, and Y-axis is the number of modules.}
\label{fig1}
\end{figure}

These collected ontologies were modularized using the approximation for module extraction for every named individual.
As discussed previously, extracting ABox module for a single individual results in a module whose signature 
is constituted by a set of individuals because of module combination. 

In Table \ref{results}, we show the statistics of maximum and average module size in terms of number of assertions 
(\#Ast.) and size of signature (\#Ind.). 
It can be observed that, in average, modules of all these ontologies are significantly smaller as compared
with the entire ABox. In some ontologies, the maximum module size is relatively large, either because there is a great number 
of assertions of a single individual, or because there are indeed intricate relationships between individuals that may affect 
classification of each other.

VICODI is a simple ontology that has no property assertions in the ABox with condition (\ref{e5}) satisfied, and thus, the signature 
of every ABox module is constituted by only a single entity. However, its maximum ABox module consists of more than 8000 
assertions for a single individual. 

For the biomedical ontologies AT and CE, their maximum ABox modules are large and complex, mainly because: 
($\romannumeral 1$) the terminological part of these ontologies is highly complex, with 33 out of 55 object 
properties either functional/inverse functional or used as restrictions for concept definition; and ($\romannumeral 2$) these ontologies 
also have single individuals that each has a great number of property assertions 
(e.g. AT has one individual with 8520 assertions). Thus, connections between these individuals by any property assertions 
satisfying condition (\ref{e5}) result in large ABox modules. 

For LUBM-1 and LUBM-2, the sizes of their maximum modules 
are between of these two categories above, due to moderate complexity of the ontologies and the fact that only 4 out of 25 object 
properties are used for concept definition. 

For DBpedia$^\star$ ontologies, though their terminological part $\mathcal{T}$ is extended from the original $\mathcal{ALF}$ to $\mathcal{SHIQ}$,
the terminology is still simple, since the extension to $\mathcal{SHIQ}$ is limited and most of its concepts are not defined on restrictions.
Nonetheless, maximum ABox modules of these ontologies are even larger, and it is mainly because of those individuals that each has 
huge number of property assertions (e.g. DBpedia$^\star\_$1 has one individual with 40773 assertions, and DBpedia$^\star\_$2 has one 
individual with 60935 assertions, etc.), 
and some of their property assertions happen to satisfy the condition (\ref{e5}) that cause a great number of module combinations.

As shown in Table \ref{smallabox}, most of the ABox modules (more than 90\%) in these ontologies, are small and simple
with no greater than ten assertions or with a single individual in signature. For modules 
with more than one individual in signature, we plot distributions of signature sizes of these ABox modules in 
Figure \ref{fig1} for LUBM-1,LUBM-2, AT, CE, DBpedia$^\star\_$1 and DBpedia$^\star\_$3. 
The X-axis gives the size range and the Y-axis gives the number of modules in 
each size range. Because of a large span and uneven distribution of module sizes, we use non-uniform size ranges, so 
that we can have a relatively simple while detailed view of distributions of small, medium, and large modules. It can be 
observed from the figure that: ($\romannumeral 1$) for all these ontologies, the majority of ABox modules are still small with 
no more than five individuals in signature, ($\romannumeral 2$) LUBMs have more medium modules that have 
signature size between 50 and 600 due to moderate ontology complexity, and ($\romannumeral 3$) the biomedical ontologies have 
almost all modules with signature size below 200 but one with very large signature (more than 1,000 individuals), and similar situations
are also found in DBpedia$^\star$ ontologies. Those large ABox modules in both biomedical and DBpedia$^\star$ ontologies
could be caused by complexity of the ontology as discussed above.

\subsection{Optimization and Comparison}

In this section,
we discuss several simple optimization techniques that can be applied to further reduce the size of an ABox module.

As discussed previously, a precise ABox module for signature $\bf S$
should consist of assertions in justifications for entailments about
individuals in $\bf S$, i.e. module-essential assertions. 
Considering there may be more than one justification for a single entailment \citep{Kalyanpur2007}, an intuition for the strategy to  
reduce sizes of ABox modules is to exclude redundant justifications for the same entailments. 

For example, consider an ontology $\mathcal{K}$ that entails 
\begin{center}
$\exists R.B \sqsubseteq A$, $B(b_1)$ and $B(b_2)$, 
\end{center}
and the ABox contains
\begin{center}
$R(a, b_1)$ and $R(a,b_2)$.
\end{center}
\noindent
To preserve the fact $A(a)$ while excluding redundant justifications for a single entailment, 
an ABox module for individual $a$ should include either $R(a, b_1)$ with $Just(B(b_1),\mathcal{K})$ or $R(a, b_2)$ with $Just(B(b_2),\mathcal{K})$, but not
both. Additionally, instead of considering all $Just(B(b_1),\mathcal{K})$s ($\it resp.$ $Just(B(b_2),\mathcal{K})$s), taking a single justification for $B(b_1)$ ($\it resp.$ for $B(b_2)$) 
should be sufficient.

Therefore, to exclude redundant justifications for the same entailments from an ABox module, beyond testing the condition (\ref{e4}), i.e.
\begin{center}
 $\mathcal{K} \models \ \geq nR.C_1 \sqcap C_3 \sqsubseteq C_2 \ \wedge \ |R^{\mathcal{K}}(a, C_1)| \geq n $
\end{center}
for every property assertion $R_0(a,b)$ when computing ${\mathcal{M}}_{\{a\}}^{\mathcal{C}}$,
it is also necessary to consider: 
\begin{description}
\item [$Case1.$] if any justification for $C_2(a)$ has already been added into the module, or 

\item [$Case2.$] if there is any justification for $C_1(b)$ that can be easily obtained. 
\end{description}
In $Case1$, if the module already preserves the fact $C_2(a)$, redundant assertions for the same entailment will not be included; 
in $Case2$, instead of adding all $Just(C_1(b),{\mathcal{K}})$s to $a$'s module, we should first consider if there is any one that is easy to compute. 
Based on these two cases, we implemented three simple optimized approximations for evaluation and comparison:
\begin{description}
\item[$\rm \bf Opt1.$] Based on $Case1$, if $C_2(a)$ can be entailed from explicit class assertions of $a$, modules of $a$ and $b$
will not be merged.

\item[$\rm \bf Opt2.$] Based on $Case2$, if $B(b)$ can be entailed from explicit class assertions of $b$, instead of merging modules of $a$ and $b$, 
only explicit class assertions of $b$ will be added into $a$'s module. 
If we preprocess every axiom in $\mathcal{T}$ to the form of $\geq nR.C_1 \sqcap C_3 \sqsubseteq C_2$ as described in the proof for Proposition \ref{p9}, this 
optimization coincides with the one that can be obtained by extending the method in \citep{Wandelt2012}.

\item[$\rm \bf Opt3.$] The combination of Opt1 and Opt2.
\end{description}

\begin{table*}[t]
\caption{Evaluation and Comparison of Optimized Approximations}
\label{opts}
\setlength{\tabcolsep}{.9mm}
\begin{center}
\resizebox{0.9\textwidth}{!}{
\begin{tabular}{c|c|c|c|c|c|c}
\hline
& \multicolumn{2}{c|}{Opt1}  & \multicolumn{2}{c|}{Opt2 \citep{Wandelt2012}}  & \multicolumn{2}{c}{Opt3}  \\ 
& \multicolumn{2}{c|}{\ } & \multicolumn{2}{c|}{} \\ 
Ontology & \shortstack{Max. \\ \#Ast./\#Ind.} & \shortstack{Avg. \\ \#Ast./\#Ind.} & \shortstack{Max. \\ \#Ast./\#Ind.} 
& \shortstack{Avg. \\ \#Ast./\#Ind.} & \shortstack{Max. \\ \#Ast./\#Ind.} & \shortstack{Avg. \\ \#Ast./\#Ind.}
\\ 

\hline
LUBM-1 & 773/2 & 6.8/1.0 & 773/2 & 6.8/1.0 & 732/1 & 6.8/1.0
\\

\hline
LUBM-2 & 773/2 & 7.1/1.0 & 773/2 & 7.1/1.0 & 732/1 & 7.1/1.0
\\

\hline
AT &  54,561/10,870 & 6.9/1.7 & 54,561/10,870 & 6.9/1.7 & 54,561/10,870 & 6.9/1.7
\\

\hline
CE & 49,914/9,315 & 7.1/1.7 &  49,914/9,315 & 7.1/1.7 & 49,914/9,315 & 7.1/1.7 
\\

\hline
DBpedia$^\star\_$1 & 52,539/13,977 & 2.9/1.1 & 87,721/15,258 & 2.9/1.1 & 228/25 & 2.7/1.0
\\

\hline
DBpedia$^\star\_$2 & 108,772/10,568 & 2.8/1.1 & 76,386/12,058 & 2.7/1.1 & 150/18 & 2.6/1.0
\\

\hline
DBpedia$^\star\_3$ & 106,344/32,063  & 3.1/1.2 & 63,092/10,385 & 2.9/1.1 & 113/29 & 2.8/1.0
\\

\hline
DBpedia$^\star\_4$ & 40,760/8,182 & 2.9/1.1 & 85,152/11,366 & 2.8/1.1 & 171/18 & 2.7/1.0
\\
\hline
\end{tabular}}
\end{center}
\end{table*}

It is obvious that the first two strategies, Opt1 and Opt2, should have no advantage over each other, and their performances will mostly depend 
on the actual ontology data to which they are applied. More precisely, if $Case1$ prevails in the ontology, Opt1 is expected to 
generate smaller ABox modules than that of Opt2, and otherwise if $Case2$ prevails. Nonetheless, Opt3 should always have better or no-worse
performance than that of Opt1 and Opt2, as it takes advantage of both of them.

In general, these three optimized approximations are expected to outperform the original approximation and produce smaller 
ABox modules, since they all employ a DL-reasoner to rule out redundant assertions for the same entailments. 
Indeed, when an ontology is simple and contains considerable redundant implications, these methods are able to efficiently reduce sizes 
of ABox modules; while on the other hand, when ontologies are complex or with less explicit redundant information, these methods may not provide significant size reductions. 

Evaluation and comparison for these three optimized approximations are shown in Table \ref{opts}. 
For LUBM ontologies, all these optimized approximations efficiently reduce the size of ABox modules (especially the maximum one), and particularly, Opt3 
reduces signature of every ABox module to include only a single
individual. For DBpedia$^\star$ ontologies, Opt1 or Opt2 achieve only
limited size reduction  
for the maximum ABox modules, while their combination, i.e. Opt3, is able to decrease the size of maximum ABox modules significantly. Nevertheless, for the 
biomedical ontologies, which are complex and may contain few explicit duplicate entailments, none of the three optimization strategies produced reductions in module size.

\subsection{Reasoning with ABox Modules}

\begin{table}[t]
\caption{Modular v.s. Complete ABox reasoning}
\label{reasoning}
\setlength{\tabcolsep}{.9mm}
\begin{center}
\resizebox{0.7\textwidth}{!}{
\begin{tabular}{c|c|c|c|c}
\hline
& \multicolumn{2}{c|}{Modular}  & \multicolumn{2}{c}{Complete}  \\ 
& \multicolumn{2}{c|}{} & \multicolumn{2}{c}{} \\ 
Ontology & \shortstack{IC Time (ms) \\ Max./Avg.} & \shortstack{IR
  Time \\ Avg.} & \shortstack{IC Time (ms) \\ Avg.} &
\shortstack{IR Time \\ Avg.}
\\ 

\hline
LUBM-1 & 17.00 / 1.93 & 33.08 (s) & 733.21 & 3.50 (h)
\\

\hline
LUBM-2 & 17.00 / 1.91 & 150.36 (s) & 9,839.22 & --
\\

\hline
AT & 3,120.00 / 344.53 & 4.09 (h) & 11,378.63 & -- 
\\

\hline
CE & 3,151.00 / 542.60 & 5.60 (h) & 10,336.61 & -- 
\\

\hline
DBpedia$^\star\_$1 & 1,326.00 / 19.10 & 1.45 (h) & 6,187.01 & --
\\

\hline
DBpedia$^\star\_$2 & 1,497.00 / 20.20 & 1.67 (h) & 6,948.20 & --
\\

\hline
DBpedia$^\star\_3$ & 3,189.00 / 19.89 & 1.41 (h) & 6,087.23 & --
\\

\hline
DBpedia$^\star\_4$ & 1,154.00 / 20.00  & 1.52 (h)  & 6,305.41 & --
\\

\hline
\end{tabular}}
\end{center}
\end{table}

In this section, we show the efficiency of ontology reasoning gained
when reasoning is based on ABox modules ($\it modular\ reasoning$), and
compare it with that of the complete ABox reasoning. Notice
however,  the purpose here is not to compare modular reasoning with
those developed optimization techniques (e.g. lazy unfolding,
and satisfiability reuse etc. \citep{Motik2007,Horrocks2007}) in
existing reasoners, since they are in totally different categories and
modular reasoning still relies on the reasoning services provided by
current state-of-art reasoners.
Instead, what we aim to show here is in fact the improvement in time
efficiency that can be achieved when using the modular reasoning for
instance checking or instance retrieval on top of existing reasoning
technologies. The reasoner we used here is HermiT \citep{Motik2007,Motik2009},
which is one of the fully-fledged OWL 2 reasoners.

For evaluation, we extract ABox modules using the Opt3 discussed in
previous section and run the reasoning on each collected
ontology: we first randomly select 10 concepts that are defined on role restrictions from the testing
ontology, and for each one of them we perform the instance checking
(on every individual in the ontology) and retrieval 
using modular reasoning and complete reasoning, respectively.  Table
\ref{reasoning} details the reasoning time for both instance checking
(IC) and instance retrieval (IR). Particularly, for instance checking
using modular reasoning, we show both maximum and average reasoning
time over the 10 test concepts, since sizes of ABox modules may vary
greatly and affect the reasoning efficiency. For instance retrieval,
we simply report the average time. Besides, we also set a threshold
(24 hours) for the time-out, and if it happens to any one of the
test, we simply put a "--" in the corresponding table entry.

Notice that, for fairness the modular reasoning for instance retrieval
performed here is not parallelized, but instead running in an arbitrary
sequential order of ABox modules.  Nevertheless, we can still see the
great improvement on time efficiency when using the modular reasoning
for instance retrieval as shown in Table \ref{reasoning}. For example,
the average time for instance checking has been reduced significantly 
from seconds down to several milliseconds when using the modular
reasoning on LUBM(s) and DBpedia$^*$(s), and instance retrieval time
on LUBM(s) (respectively on DBpedia$^*$(s)) has been reduced from
several hours (respectively several days) down to seconds
(respectively less than two hours); even for those biomedical ontologies, 
the time for instance retrieval is also reduced from more than 130 hours
down to less than 6 hours. The reason behind all these improvements is
simply that the complexity of the tableau-based reasoning algorithms
is up to exponential time w.r.t. the data size for for
$\mathcal{SHIQ}$ \citep{Donini2007,Tobies2001}; once the data size is cut down,
the reasoning time could be reduced significantly.

One point to note here is that, when using the modular reasoning for
answering (conjunctive) queries, the time for instance retrieval in fact
should plus the time that is spent for modularizing the ABox.
Nevertheless, as we can see from Table \ref{results},  the overhead on
average is only around one millisecond for instance checking of each named
individual.

\section{Discussion and Outlook}

In this paper, we have proposed a formal definition of
an ABox module, such that each module ensures complete preservation of
logical entailments for a give set of individuals. Utilizing this
property, scalable object queries over big ontologies can be accomplished,
by either conducting an isolated reasoning on a single ABox
module when querying information about a particular group of
individuals, or by distributing independent reasoning tasks on ABox
modules into a cluster of computers when performing instance retrieval
or answering conjunctive queries.
To put more restrictions on an ABox module, we further defined
the notion of module-essentiality, which gives a formal criteria for ABox assertions 
to be semantically related (or unrelated) to a given signature and
could be useful for computation of a precise ABox module. 

To extract an ABox module, we presented a theoretical approach,
which is only aimed to give an exposition of the problem from a theoretical perspective
and to provide strategies that are provable theoretically sound.
For applicability in realistic ontologies, we provided a
simple and tractable approximation, which is straightforward and easy to implement,
but may include many unrelated assertions when the ontology is
complex, thus resulting in modules that are larger than desired. 
The undesired bulk of ABox modules could be caused by: ($\romannumeral 1$)
the syntactic approximation (i.e. (\ref{e5})) for semantic conditions (i.e. (\ref{e4})) that
could result in many false positives and hence cause the unnecessary
combination of ABox modules; ($\romannumeral 2$) the simple
approximation of individual equalities, which only checks syntactically
the possibility of an individual to be an instance of some concept $\leq n R.C$
and the approximated number of its $R$-neighbors that should be
no less than the associated cardinality required for individual identification;
and  ($\romannumeral 3$) last but not least, the intrinsic complexity of ontologies
that requires assertions to be grouped to preserve complete logical entailments.

Strategies for optimization can provide significant reduction in
module size under some conditions, as has been shown in the results
section. However, for highly complex ontologies like those biomedical
ones, more advanced optimization techniques are demanded. 
One of the directions for future optimization is obviously to provide more
rigorous approximations for condition (\ref{e4}) and individual
equality, and a possible solution is to add affordable semantical
verification of individual classifications, which may not only
prevent unnecessary module combination but also rule out false individual
equalities. Progress in this direction has already been made in our current projects.


\bibliographystyle{model2-names} 

\bibliography{aboxmodular}

\begin{thebibliography}{50}
\expandafter\ifx\csname natexlab\endcsname\relax\def\natexlab#1{#1}\fi
\expandafter\ifx\csname url\endcsname\relax
  \def\url#1{\texttt{#1}}\fi
\expandafter\ifx\csname urlprefix\endcsname\relax\def\urlprefix{URL }\fi
\providecommand{\eprint}[2][]{\url{#2}}
\providecommand{\bibinfo}[2]{#2}
\ifx\xfnm\relax \def\xfnm[#1]{\unskip,\space#1}\fi
\bibitem[{Auer et~al.(2007)Auer, Bizer, Kobilarov, Lehmann, Cyganiak and
  Ives}]{Auer2007}
\bibinfo{author}{Auer, S.}, \bibinfo{author}{Bizer, C.},
  \bibinfo{author}{Kobilarov, G.}, \bibinfo{author}{Lehmann, J.},
  \bibinfo{author}{Cyganiak, R.}, \bibinfo{author}{Ives, Z.},
  \bibinfo{year}{2007}.
\newblock \bibinfo{title}{Dbpedia: A nucleus for a web of open data}, in:
  \bibinfo{booktitle}{Proceedings of ISWC}, \bibinfo{publisher}{Springer}. pp.
  \bibinfo{pages}{722--735}.
\bibitem[{Baader et~al.(2007)Baader, Calvanese, McGuinness, Nardi and
  Patel-Schneider}]{Baader2007}
\bibinfo{editor}{Baader, F.}, \bibinfo{editor}{Calvanese, D.},
  \bibinfo{editor}{McGuinness, D.}, \bibinfo{editor}{Nardi, D.},
  \bibinfo{editor}{Patel-Schneider, P.} (Eds.), \bibinfo{year}{2007}.
\newblock \bibinfo{title}{The Description Logic Handbook: Theory,
  Implementation, and Applications}.
\newblock \bibinfo{publisher}{Cambridge University Press}.
\bibitem[{Bechhofer et~al.(2004)Bechhofer, van Harmelen, Hendler, Horrocks,
  McGuinness, Patel-Schneider and Stein}]{owl}
\bibinfo{author}{Bechhofer, S.}, \bibinfo{author}{van Harmelen, F.},
  \bibinfo{author}{Hendler, J.}, \bibinfo{author}{Horrocks, I.},
  \bibinfo{author}{McGuinness, D.L.}, \bibinfo{author}{Patel-Schneider, P.F.},
  \bibinfo{author}{Stein, L.A.}, \bibinfo{year}{2004}.
\newblock \bibinfo{title}{Owl web ontology language}.
\newblock \bibinfo{howpublished}{http://www.w3.org/TR/owl-ref/}.
\bibitem[{Bhatt et~al.(2009)Bhatt, Rahayu, Soni and Wouters}]{Bhatt2009}
\bibinfo{author}{Bhatt, M.}, \bibinfo{author}{Rahayu, W.},
  \bibinfo{author}{Soni, S.}, \bibinfo{author}{Wouters, C.},
  \bibinfo{year}{2009}.
\newblock \bibinfo{title}{Ontology driven semantic profiling and retrieval in
  medical information systems}.
\newblock \bibinfo{journal}{Journal of Web semantics} \bibinfo{volume}{7},
  \bibinfo{pages}{317--331}.
\bibitem[{Calvanese et~al.(2007)Calvanese, De~Giacomo, Lembo, Lenzerini and
  Rosati}]{Calvanese2007}
\bibinfo{author}{Calvanese, D.}, \bibinfo{author}{De~Giacomo, G.},
  \bibinfo{author}{Lembo, D.}, \bibinfo{author}{Lenzerini, M.},
  \bibinfo{author}{Rosati, R.}, \bibinfo{year}{2007}.
\newblock \bibinfo{title}{Tractable reasoning and efficient query answering in
  description logics: The {DL-Lite} family}.
\newblock \bibinfo{journal}{Journal of Automated reasoning}
  \bibinfo{volume}{39}, \bibinfo{pages}{385--429}.
\bibitem[{Cuenca~Grau et~al.(2007a)Cuenca~Grau, Horrocks, Kazakov and
  Sattler}]{Grau2007a}
\bibinfo{author}{Cuenca~Grau, B.}, \bibinfo{author}{Horrocks, I.},
  \bibinfo{author}{Kazakov, Y.}, \bibinfo{author}{Sattler, U.},
  \bibinfo{year}{2007}a.
\newblock \bibinfo{title}{Just the right amount: Extracting modules from
  ontologies}, in: \bibinfo{booktitle}{Proceedings of WWW'07}, pp.
  \bibinfo{pages}{717--726}.
\bibitem[{Cuenca~Grau et~al.(2007b)Cuenca~Grau, Horrocks, Kazakov and
  Sattler}]{Grau2007b}
\bibinfo{author}{Cuenca~Grau, B.}, \bibinfo{author}{Horrocks, I.},
  \bibinfo{author}{Kazakov, Y.}, \bibinfo{author}{Sattler, U.},
  \bibinfo{year}{2007}b.
\newblock \bibinfo{title}{A logical framework for modularity of ontologies},
  in: \bibinfo{booktitle}{Proceedings of International Joint Conference on
  Artificial Intelligence (IJCAI)}, \bibinfo{publisher}{AAAI}. pp.
  \bibinfo{pages}{298--304}.
\bibitem[{Cuenca~Grau et~al.(2008)Cuenca~Grau, Horrocks, Motik, Parsia,
  Patel-Schneider and Sattler}]{Grau2008}
\bibinfo{author}{Cuenca~Grau, B.}, \bibinfo{author}{Horrocks, I.},
  \bibinfo{author}{Motik, B.}, \bibinfo{author}{Parsia, B.},
  \bibinfo{author}{Patel-Schneider, P.}, \bibinfo{author}{Sattler, U.},
  \bibinfo{year}{2008}.
\newblock \bibinfo{title}{{OWL} 2: The next step for owl}.
\newblock \bibinfo{journal}{Journal of Web Semantics} \bibinfo{volume}{6},
  \bibinfo{pages}{309--322}.
\bibitem[{Dean and Ghemawat(2008)}]{Dean2008}
\bibinfo{author}{Dean, J.}, \bibinfo{author}{Ghemawat, S.},
  \bibinfo{year}{2008}.
\newblock \bibinfo{title}{Mapreduce: simplified data processing on large
  clusters}.
\newblock \bibinfo{journal}{Communications of the ACM} \bibinfo{volume}{51},
  \bibinfo{pages}{107--113}.
\bibitem[{Demir et~al.(2010)Demir, Cary, Paley, Fukuda, Lemer, Vastrik, Wu,
  D'Eustachio, Schaefer, Luciano et~al.}]{Demir2010}
\bibinfo{author}{Demir, E.}, \bibinfo{author}{Cary, M.},
  \bibinfo{author}{Paley, S.}, \bibinfo{author}{Fukuda, K.},
  \bibinfo{author}{Lemer, C.}, \bibinfo{author}{Vastrik, I.},
  \bibinfo{author}{Wu, G.}, \bibinfo{author}{D'Eustachio, P.},
  \bibinfo{author}{Schaefer, C.}, \bibinfo{author}{Luciano, J.}, et~al.,
  \bibinfo{year}{2010}.
\newblock \bibinfo{title}{The biopax community standard for pathway data
  sharing}.
\newblock \bibinfo{journal}{Nature biotechnology} \bibinfo{volume}{28},
  \bibinfo{pages}{935--942}.
\bibitem[{Donini(2007)}]{Donini2007}
\bibinfo{author}{Donini, F.}, \bibinfo{year}{2007}.
\newblock \bibinfo{title}{The Description Logic Handbook: Theory,
  Implementation, and Applications}. \bibinfo{publisher}{Cambridge University
  Press}. chapter \bibinfo{chapter}{Complexity of Reasoning}.
\bibitem[{Donini and Era(1992)}]{Donini1992}
\bibinfo{author}{Donini, F.}, \bibinfo{author}{Era, A.}, \bibinfo{year}{1992}.
\newblock \bibinfo{title}{Most specific concepts for knowledge bases with
  incomplete information}, in: \bibinfo{booktitle}{Proceedings of CIKM},
  \bibinfo{address}{Baltimor, MD}. pp. \bibinfo{pages}{545--551}.
\bibitem[{Donini et~al.(1994)Donini, Lenzerini, Nardi and Schaerf}]{Donini1994}
\bibinfo{author}{Donini, F.}, \bibinfo{author}{Lenzerini, M.},
  \bibinfo{author}{Nardi, D.}, \bibinfo{author}{Schaerf, A.},
  \bibinfo{year}{1994}.
\newblock \bibinfo{title}{Deduction in concept languages: From subsumption to
  instance checking}.
\newblock \bibinfo{journal}{Journal of logic and computation}
  \bibinfo{volume}{4}, \bibinfo{pages}{423--452}.
\bibitem[{Du and Shen(2007)}]{Du2007}
\bibinfo{author}{Du, J.}, \bibinfo{author}{Shen, Y.}, \bibinfo{year}{2007}.
\newblock \bibinfo{title}{Partitioning aboxes based on converting {DL} to plain
  datalog}, in: \bibinfo{booktitle}{Proceedings of Description Logics (DL-07)},
  \bibinfo{organization}{Citeseer}. pp. \bibinfo{pages}{251--258}.
\bibitem[{Fokoue et~al.(2006)Fokoue, Kershenbaum, Ma, Schonberg and
  Srinivas}]{Fokoue2006}
\bibinfo{author}{Fokoue, A.}, \bibinfo{author}{Kershenbaum, A.},
  \bibinfo{author}{Ma, L.}, \bibinfo{author}{Schonberg, E.},
  \bibinfo{author}{Srinivas, K.}, \bibinfo{year}{2006}.
\newblock \bibinfo{title}{The summary abox: Cutting ontologies down to size},
  in: \bibinfo{booktitle}{Proceedings of ISWC-06},
  \bibinfo{publisher}{Springer}. pp. \bibinfo{pages}{343--356}.
\bibitem[{Glimm et~al.(2008)Glimm, Horrocks, Lutz and Sattler}]{Glimm2008}
\bibinfo{author}{Glimm, B.}, \bibinfo{author}{Horrocks, I.},
  \bibinfo{author}{Lutz, C.}, \bibinfo{author}{Sattler, U.},
  \bibinfo{year}{2008}.
\newblock \bibinfo{title}{Conjunctive query answering for the description
  logic}.
\newblock \bibinfo{journal}{Journal Artificial Intelligence Research}
  \bibinfo{volume}{31}, \bibinfo{pages}{157--204}.
\bibitem[{Grosof et~al.(2003)Grosof, Horrocks, Volz and Decker}]{Grosof2003}
\bibinfo{author}{Grosof, B.}, \bibinfo{author}{Horrocks, I.},
  \bibinfo{author}{Volz, R.}, \bibinfo{author}{Decker, S.},
  \bibinfo{year}{2003}.
\newblock \bibinfo{title}{Description logic programs: Combining logic programs
  with description logic}, in: \bibinfo{booktitle}{Proceedings of WWW},
  \bibinfo{organization}{ACM}. pp. \bibinfo{pages}{48--57}.
\bibitem[{Guo and Heflin(2006)}]{Guo2006}
\bibinfo{author}{Guo, Y.}, \bibinfo{author}{Heflin, J.}, \bibinfo{year}{2006}.
\newblock \bibinfo{title}{A scalable approach for partitioning owl knowledge
  bases}, in: \bibinfo{booktitle}{International Workshop on Scalable Semantic
  Web Knowledge Base Systems (SSWS)}, \bibinfo{publisher}{Springer}. pp.
  \bibinfo{pages}{636--641}.
\bibitem[{Guo et~al.(2005)Guo, Pan and Heflin}]{Guo2005}
\bibinfo{author}{Guo, Y.}, \bibinfo{author}{Pan, Z.}, \bibinfo{author}{Heflin,
  J.}, \bibinfo{year}{2005}.
\newblock \bibinfo{title}{{LUBM}: A benchmark for owl knowledge base systems.}
\newblock \bibinfo{journal}{Journal of Web Semantics} \bibinfo{volume}{3},
  \bibinfo{pages}{158--182}.
\bibitem[{Haarslev and M{\"o}ller(2001)}]{Haarslev2001}
\bibinfo{author}{Haarslev, V.}, \bibinfo{author}{M{\"o}ller, R.},
  \bibinfo{year}{2001}.
\newblock \bibinfo{title}{{RACER} system description}, in:
  \bibinfo{booktitle}{Proceedings of the First International Joint Conference
  on Automated Reasoning}, \bibinfo{publisher}{Springer}. pp.
  \bibinfo{pages}{701--705}.
\bibitem[{Haarslev and M{\"o}ller(2002)}]{Haarslev2002}
\bibinfo{author}{Haarslev, V.}, \bibinfo{author}{M{\"o}ller, R.},
  \bibinfo{year}{2002}.
\newblock \bibinfo{title}{Optimization strategies for instance retrieval}, in:
  \bibinfo{booktitle}{Proc. International Workshop on Description Logics (DL)}.
\bibitem[{Haghighi et~al.(2012)Haghighi, Burstein, Zaslavsky and
  Arbon}]{Haghighi2012}
\bibinfo{author}{Haghighi, P.}, \bibinfo{author}{Burstein, F.},
  \bibinfo{author}{Zaslavsky, A.}, \bibinfo{author}{Arbon, P.},
  \bibinfo{year}{2012}.
\newblock \bibinfo{title}{Development and evaluation of ontology for
  intelligent decision support in medical emergency management for mass
  gatherings}.
\newblock \bibinfo{journal}{Decision Support Systems} .
\bibitem[{Hitzler et~al.(2009)Hitzler, Kr{\"o}tzsch and Rudolph}]{FOST}
\bibinfo{author}{Hitzler, P.}, \bibinfo{author}{Kr{\"o}tzsch, M.},
  \bibinfo{author}{Rudolph, S.}, \bibinfo{year}{2009}.
\newblock \bibinfo{title}{Foundations of Semantic Web Technologies}.
\newblock \bibinfo{publisher}{Chapman Hall CRC}.
\bibitem[{Horrocks(2007)}]{Horrocks2007}
\bibinfo{author}{Horrocks, I.}, \bibinfo{year}{2007}.
\newblock \bibinfo{title}{The Description Logic Handbook: Theory,
  Implementation, and Applications}. \bibinfo{publisher}{Cambridge University
  Press}. chapter \bibinfo{chapter}{9 Implementation and Optimization
  Techniques}.
\bibitem[{Horrocks(2008)}]{Horrocks2008}
\bibinfo{author}{Horrocks, I.}, \bibinfo{year}{2008}.
\newblock \bibinfo{title}{Ontologies and the semantic web}.
\newblock \bibinfo{journal}{Communications of the ACM} \bibinfo{volume}{51},
  \bibinfo{pages}{58--67}.
\bibitem[{Horrocks et~al.(2004)Horrocks, Li, Turi and Bechhofer}]{Horrocks2004}
\bibinfo{author}{Horrocks, I.}, \bibinfo{author}{Li, L.},
  \bibinfo{author}{Turi, D.}, \bibinfo{author}{Bechhofer, S.},
  \bibinfo{year}{2004}.
\newblock \bibinfo{title}{The instance store: {DL} reasoning with large numbers
  of individuals}, in: \bibinfo{booktitle}{Proceedings of the Description Logic
  Workshop (DL-2004)}, pp. \bibinfo{pages}{31--40}.
\bibitem[{Horrocks and Patel-Schneider(2004)}]{Horrocks2004b}
\bibinfo{author}{Horrocks, I.}, \bibinfo{author}{Patel-Schneider, P.},
  \bibinfo{year}{2004}.
\newblock \bibinfo{title}{Reducing owl entailment to description logic
  satisfiability}.
\newblock \bibinfo{journal}{Journal of Web Semantics} \bibinfo{volume}{1},
  \bibinfo{pages}{345--357}.
\bibitem[{Horrocks et~al.(2003)Horrocks, Patel-Schneider and
  Van~Harmelen}]{Horrocks2003}
\bibinfo{author}{Horrocks, I.}, \bibinfo{author}{Patel-Schneider, P.},
  \bibinfo{author}{Van~Harmelen, F.}, \bibinfo{year}{2003}.
\newblock \bibinfo{title}{From {SHIQ} and {RDF} to {OWL}: The making of a web
  ontology language}.
\newblock \bibinfo{journal}{Journal of Web Semantics} \bibinfo{volume}{1},
  \bibinfo{pages}{7--26}.
\bibitem[{Horrocks et~al.(2000)Horrocks, Sattler and Tobies}]{Horrocks2000}
\bibinfo{author}{Horrocks, I.}, \bibinfo{author}{Sattler, U.},
  \bibinfo{author}{Tobies, S.}, \bibinfo{year}{2000}.
\newblock \bibinfo{title}{Reasoning with individuals for the description logic
  {SHIQ}}, in: \bibinfo{booktitle}{Proceedings of Conference on Automated
  Deduction (CADE)}, \bibinfo{publisher}{Springer}. pp.
  \bibinfo{pages}{482--496}.
\bibitem[{Horrocks and Tessaris(2000)}]{Horrocks2000b}
\bibinfo{author}{Horrocks, I.}, \bibinfo{author}{Tessaris, S.},
  \bibinfo{year}{2000}.
\newblock \bibinfo{title}{A conjunctive query language for description logic
  aboxes}, in: \bibinfo{booktitle}{Proceedings of AAAI}, pp.
  \bibinfo{pages}{399--404}.
\bibitem[{Huang et~al.(2005)Huang, Van~Harmelen and Teije}]{Huang2005}
\bibinfo{author}{Huang, Z.}, \bibinfo{author}{Van~Harmelen, F.},
  \bibinfo{author}{Teije, A.}, \bibinfo{year}{2005}.
\newblock \bibinfo{title}{Reasoning with inconsistent ontologies}, in:
  \bibinfo{booktitle}{Proceedings of International Joint Conference on
  Artificial Intelligence (IJCAI)}, \bibinfo{publisher}{AAAI}. p.
  \bibinfo{pages}{454}.
\bibitem[{Hustadt et~al.(2004)Hustadt, Motik and Sattler}]{Hustadt2004}
\bibinfo{author}{Hustadt, U.}, \bibinfo{author}{Motik, B.},
  \bibinfo{author}{Sattler, U.}, \bibinfo{year}{2004}.
\newblock \bibinfo{title}{Reducing {SHIQ-} description logic to disjunctive
  datalog programs}, in: \bibinfo{booktitle}{Proceedings of KR}, pp.
  \bibinfo{pages}{152--162}.
\bibitem[{Iqbal et~al.(2011)Iqbal, Shepherd and Abidi}]{Iqbal2011}
\bibinfo{author}{Iqbal, A.}, \bibinfo{author}{Shepherd, M.},
  \bibinfo{author}{Abidi, S.}, \bibinfo{year}{2011}.
\newblock \bibinfo{title}{An ontology-based electronic medical record for
  chronic disease management}, in: \bibinfo{booktitle}{Proceedings of Hawaii
  International Conference on System Sciences (HICSS)},
  \bibinfo{organization}{IEEE}. pp. \bibinfo{pages}{1--10}.
\bibitem[{Kalyanpur et~al.(2007)Kalyanpur, Parsia, Horridge and
  Sirin}]{Kalyanpur2007}
\bibinfo{author}{Kalyanpur, A.}, \bibinfo{author}{Parsia, B.},
  \bibinfo{author}{Horridge, M.}, \bibinfo{author}{Sirin, E.},
  \bibinfo{year}{2007}.
\newblock \bibinfo{title}{Finding all justifications of {OWL DL} entailments},
  in: \bibinfo{booktitle}{Proceedings of International Semantic Web Conference
  (ISWC)}, pp. \bibinfo{pages}{267--280}.
\bibitem[{Lee et~al.(2008)Lee, Wang and Chen}]{Lee2008}
\bibinfo{author}{Lee, C.}, \bibinfo{author}{Wang, M.}, \bibinfo{author}{Chen,
  J.}, \bibinfo{year}{2008}.
\newblock \bibinfo{title}{Ontology-based intelligent decision support agent for
  cmmi project monitoring and control}.
\newblock \bibinfo{journal}{International Journal of Approximate Reasoning}
  \bibinfo{volume}{48}, \bibinfo{pages}{62--76}.
\bibitem[{Motik and Sattler(2006)}]{Motik2006}
\bibinfo{author}{Motik, B.}, \bibinfo{author}{Sattler, U.},
  \bibinfo{year}{2006}.
\newblock \bibinfo{title}{A comparison of reasoning techniques for querying
  large description logic aboxes}, in: \bibinfo{booktitle}{Proceedings of
  LPAR'06}, pp. \bibinfo{pages}{227--241}.
\bibitem[{Motik et~al.(2005)Motik, Sattler and Studer}]{Motik2005}
\bibinfo{author}{Motik, B.}, \bibinfo{author}{Sattler, U.},
  \bibinfo{author}{Studer, R.}, \bibinfo{year}{2005}.
\newblock \bibinfo{title}{Query answering for {OWL-DL} with rules}.
\newblock \bibinfo{journal}{Journal of Web Semantics} \bibinfo{volume}{3},
  \bibinfo{pages}{41--60}.
\bibitem[{Motik et~al.(2007)Motik, Shearer and Horrocks}]{Motik2007}
\bibinfo{author}{Motik, B.}, \bibinfo{author}{Shearer, R.},
  \bibinfo{author}{Horrocks, I.}, \bibinfo{year}{2007}.
\newblock \bibinfo{title}{{Optimized Reasoning in Description Logics using
  Hypertableaux}}, in: \bibinfo{editor}{Pfenning, F.} (Ed.),
  \bibinfo{booktitle}{Proceedings of Conference on Automated Deduction (CADE)},
  \bibinfo{publisher}{Springer}, \bibinfo{address}{Bremen, Germany}. pp.
  \bibinfo{pages}{67--83}.
\bibitem[{Motik et~al.(2009)Motik, Shearer and Horrocks}]{Motik2009}
\bibinfo{author}{Motik, B.}, \bibinfo{author}{Shearer, R.},
  \bibinfo{author}{Horrocks, I.}, \bibinfo{year}{2009}.
\newblock \bibinfo{title}{Hypertableau reasoning for description logics}.
\newblock \bibinfo{journal}{Journal of Artificial Intelligence Research}
  \bibinfo{volume}{36}, \bibinfo{pages}{165--228}.
\bibitem[{Nebel(1990)}]{Nebel1990}
\bibinfo{author}{Nebel, B.}, \bibinfo{year}{1990}.
\newblock \bibinfo{title}{Reasoning and revision in hybrid representation
  systems}. volume \bibinfo{volume}{422}.
\newblock \bibinfo{publisher}{Springer-Verlag Germany}.
\bibitem[{Ortiz et~al.(2008)Ortiz, Calvanese and Eiter}]{Ortiz2008}
\bibinfo{author}{Ortiz, M.}, \bibinfo{author}{Calvanese, D.},
  \bibinfo{author}{Eiter, T.}, \bibinfo{year}{2008}.
\newblock \bibinfo{title}{Data complexity of query answering in expressive
  description logics via tableaux}.
\newblock \bibinfo{journal}{Journal of Automated Reasoning}
  \bibinfo{volume}{41}, \bibinfo{pages}{61--98}.
\bibitem[{Royer and Quantz(1993)}]{Royer1993}
\bibinfo{author}{Royer, V.}, \bibinfo{author}{Quantz, J.},
  \bibinfo{year}{1993}.
\newblock \bibinfo{title}{Deriving Inference Rules for Description Logics: a
  Rewriting Approach into Sequent Calculi}.
\newblock \bibinfo{type}{Technical Report}. Technische Universitaet Berlin.
  \bibinfo{address}{Germany}.
\bibitem[{Schaerf(1994)}]{Schaerf1994}
\bibinfo{author}{Schaerf, A.}, \bibinfo{year}{1994}.
\newblock \bibinfo{title}{Reasoning with individuals in concept languages}.
\newblock \bibinfo{journal}{Data and Knowledge Engineering}
  \bibinfo{volume}{13}, \bibinfo{pages}{141--176}.
\bibitem[{Schmidt-Schau{\ss} and Smolka(1991)}]{Schmidt1991}
\bibinfo{author}{Schmidt-Schau{\ss}, M.}, \bibinfo{author}{Smolka, G.},
  \bibinfo{year}{1991}.
\newblock \bibinfo{title}{Attributive concept descriptions with complements}.
\newblock \bibinfo{journal}{Artificial intelligence} \bibinfo{volume}{48},
  \bibinfo{pages}{1--26}.
\bibitem[{Sirin et~al.(2007)Sirin, Parsia, Grau, Kalyanpur and
  Katz}]{Sirin2007}
\bibinfo{author}{Sirin, E.}, \bibinfo{author}{Parsia, B.},
  \bibinfo{author}{Grau, B.C.}, \bibinfo{author}{Kalyanpur, A.},
  \bibinfo{author}{Katz, Y.}, \bibinfo{year}{2007}.
\newblock \bibinfo{title}{Pellet: A practical owl-dl reasoner}.
\newblock \bibinfo{journal}{Journal of Web Semantics} \bibinfo{volume}{5},
  \bibinfo{pages}{51 -- 53}.
\bibitem[{Tobies(2001)}]{Tobies2001}
\bibinfo{author}{Tobies, S.}, \bibinfo{year}{2001}.
\newblock \bibinfo{title}{Complexity Results and Practical Algorithms for
  Logics in Knowledge Representation}.
\newblock Ph.D. thesis. RWTH Aachen.
\bibitem[{Tsarkov and Horrocks(2006)}]{Tsarkov2006}
\bibinfo{author}{Tsarkov, D.}, \bibinfo{author}{Horrocks, I.},
  \bibinfo{year}{2006}.
\newblock \bibinfo{title}{{FaCT}++ description logic reasoner: System
  description}.
\newblock \bibinfo{journal}{Automated Reasoning} , \bibinfo{pages}{292--297}.
\bibitem[{Visser et~al.(2011)Visser, Abeyruwan, Vempati, Smith, Lemmon and
  Sch{\"u}rer}]{Visser2011}
\bibinfo{author}{Visser, U.}, \bibinfo{author}{Abeyruwan, S.},
  \bibinfo{author}{Vempati, U.}, \bibinfo{author}{Smith, R.},
  \bibinfo{author}{Lemmon, V.}, \bibinfo{author}{Sch{\"u}rer, S.},
  \bibinfo{year}{2011}.
\newblock \bibinfo{title}{{BioAssay Ontology (BAO)}: a semantic description of
  bioassays and high-throughput screening results}.
\newblock \bibinfo{journal}{BMC bioinformatics} \bibinfo{volume}{12},
  \bibinfo{pages}{257}.
\bibitem[{Wandelt and M{\"o}ller(2012)}]{Wandelt2012}
\bibinfo{author}{Wandelt, S.}, \bibinfo{author}{M{\"o}ller, R.},
  \bibinfo{year}{2012}.
\newblock \bibinfo{title}{Towards abox modularization of semi-expressive
  description logics}.
\newblock \bibinfo{journal}{Applied Ontology} \bibinfo{volume}{7},
  \bibinfo{pages}{133--167}.
\bibitem[{Williams et~al.(2010)Williams, Weaver, Atre and
  Hendler}]{Williams2010}
\bibinfo{author}{Williams, G.}, \bibinfo{author}{Weaver, J.},
  \bibinfo{author}{Atre, M.}, \bibinfo{author}{Hendler, J.},
  \bibinfo{year}{2010}.
\newblock \bibinfo{title}{Scalable reduction of large datasets to interesting
  subsets}.
\newblock \bibinfo{journal}{Journal of Web Semantics} \bibinfo{volume}{8},
  \bibinfo{pages}{365--373}.

\end{thebibliography}

\end{document}